\documentclass[twoside]{article}

\usepackage[accepted]{aistats2023}

\usepackage{amsmath, amsthm, amsfonts,bm}

\def\eqref#1{equation~\ref{#1}}

\def\floor#1{\lfloor #1 \rfloor}
\def\1{\bm{1}}

\def\vu{{\bm{u}}}

\DeclareMathAlphabet{\mathsfit}{\encodingdefault}{\sfdefault}{m}{sl}
\SetMathAlphabet{\mathsfit}{bold}{\encodingdefault}{\sfdefault}{bx}{n}

\def\gA{{\mathcal{A}}}

\def\sA{{\mathbb{A}}}

\def\sI{{\mathbb{I}}}

\def\sN{{\mathbb{N}}}

\def\sP{{\mathbb{P}}}

\def\sR{{\mathbb{R}}}

\newcommand{\E}{\mathbb{E}}

\newcommand{\KL}{D_{\mathrm{KL}}}

\theoremstyle{plain}
\newtheorem{theorem}{Theorem} 
\newtheorem{corollary}[theorem]{Corollary}
\newtheorem{lemma}[theorem]{Lemma}

\theoremstyle{definition}
\newtheorem{definition}[theorem]{Definition} 
\usepackage{common}
\usepackage{xcolor}
\usepackage{hyperref}
\usepackage{multirow}

\usepackage[round]{natbib}

\begin{document}

\twocolumn[

\aistatstitle{Adaptation to Misspecified Kernel Regularity in Kernelised Bandits}

\aistatsauthor{ Yusha Liu \And Aarti Singh }

\aistatsaddress{ Carnegie Mellon University\\yushal@cs.cmu.edu \And Carnegie Mellon University\\aarti@cs.cmu.edu } 
]

\begin{abstract}
  In continuum-armed bandit problems where the underlying function resides in a reproducing kernel Hilbert space (RKHS), namely, the kernelised bandit problems, an important open problem remains of how well learning algorithms can adapt if the regularity of the associated kernel function is unknown. In this work, we study adaptivity to \highlight{the regularity of translation-invariant kernels}, which is characterized by the decay rate of the Fourier transformation of the kernel, \highlight{in the bandit setting.} We derive an adaptivity lower bound, proving that it is impossible to simultaneously achieve optimal cumulative regret in a pair of RKHSs with different regularities. To verify the tightness of this lower bound, we show that an existing bandit model selection algorithm applied with minimax non-adaptive kernelised bandit algorithms matches the lower bound in dependence of $T$, the total number of steps, except for log factors. By \highlight{filling in the regret bounds for adaptivity between RKHSs}, we connect the statistical difficulty for adaptivity in continuum-armed bandits in three fundamental types of function spaces: RKHS, Sobolev space, and H\"older space.
\end{abstract}

\section{Introduction}\label{sec: intro}
\highlight{We consider the problem of continuum-armed bandit, a sequential decision-making problem,} where the goal of a learning algorithm is the optimization of a black-box reward function, by selecting query points and eliciting rewards from the underlying function sequentially. The performance of algorithms is measured by the cumulative regret, which is the sum of differences between the maximum of the underlying function and the reward incurred by the learning algorithm across all the time steps. Optimizing cumulative regret requires from the learning algorithms a delicate exploration-exploitation tradeoff. The learning algorithm needs to simultaneously exploit high-reward regions and explore uncertain regions. 
The exploration-exploitation tradeoff is often dependent on \highlight{complexity} of the function space to which the reward function belongs. 
In most theoretical analyses of cumulative regret of algorithms, \highlight{complexity} of the function space is assumed to be known.  
Many studies use this assumption to design algorithms that achieve minimax optimal performance when the function space is known, for example, for linear functions~\citep{dani2008stochastic,abbasi2011improved}, functions residing in reproducing kernel Hilbert spaces (RKHS)~\citep{valko2013finite,janz2020bandit} or drawn from Gaussian Processes~\citep{srinivas2009gaussian, chowdhury2017kernelized}, as well as neural networks functions~\citep{zhou2020neural,kassraie2021neural}.

However, despite the theoretical convenience, it is not always realistic to assume access to the underlying function space. For this reason, some recent works in continuum-armed bandits have started to develop adaptive algorithms for when the function space is misspecified (see Section~\ref{sec: related work} for a summary of related works). The best possible performance of adaptive algorithms 
is equivalent to algorithms that know the parameter. An algorithm that simultaneously achieves minimax cumulative regret rates without access to the parameter is said to achieve minimax adaptivity. While minimax adaptivity is possible under the simple regret minimization setting, recent works have proved that it is not always achievable for cumulative regret minimzation~\citep{locatelli2018adaptivity}, due to the exploration-exploitation dilemma.

When the reward function resides in an RKHS induced of some kernel function~$\kernel$, the problem also is referred to as kernelised bandit. In this work, we focus on an important and open problem in adaptivity in kernelised bandits,  precisely, adaptivity to unknown kernel regularity.
Recently, there has been a line of theoretical works that study adaptivity under the kernelised bandit setting, such as adaptivity to the length scale parameter and the RKHS norm~\citep{berkenkamp2019no} for a given kernel, and adaptivity to $\epsilon$-misspecification, where the underlying function is $\epsilon$-approximated by functions in an RKHS~\citep{bogunovic2021misspecified}. 
To the best of our knowledge, the work of \citet{kassraie2022meta} is most closely related to our setting. They consider the setting where the underlying function lies in an RKHS but the kernel is unknown. \citet{kassraie2022meta} assume that the kernel is a sparse combination of known base kernels and design algorithms with sublinear regret guarantees under this assumption. A more detailed discussion of the prior works on adaptivity in kernelised bandit is continued in Section~\ref{sec: related work}. 

Adaptivity of any algorithm with respect to the explicit regularity of the kernel function $\kernel$, however, remains an unsolved problem. We characterize the regularity of $\kernel$ using a general notion: the decay rate of the Fourier transform of $\kernel$ (Section~\ref{sec: problem setting}). In contrast to, for example, adapting to the RKHS norm which measures the smoothness of a function with respect to a fixed kernel, we adapt to the regularity of kernels which controls the differentiability of functions in the associated RKHSs. 
The kernel regularity thus determines the statistical complexity of the associated learning problem in a more fundamental way. \highlight{In estimation, optimization (including simple regret minimization)~\citep{bull2011convergence} and cumulative regret minimization tasks~\citep{srinivas2009gaussian, kandasamy2019myopic, janz2020bandit}, the kernel regularity affects the minimax regret rate through exponential dependence on $T$, {as opposed to the RKHS norm which only affects the rate polynomially}}. We focus on this fundamental problem of how well bandit algorithms can adapt to the unknown kernel regularity.

The contributions of this work are summarized as follows:
\begin{enumerate}
    \item We derive the first lower bound on adaptivity to kernel regularity, expressed in terms of the kernel Fourier transformation decay rate, for kernelised bandit problems. This lower bound serves as an impossibility result, that no algorithms can simultaneously achieve minimax optimal performance in RKHSs with different regularities.
    \item For RKHSs of the \matern~family~\citep{matern1960spatial} of kernels, we prove that CORRAL~\citep{pacchiano2020model}, an existing model selection algorithm, applied with (non-adaptive) minimax optimal kernelised bandit algorithms, matches the adaptivity lower bound\footnote{Except for log factors.} in the dependence on $T$. In contrast, another model selection algorithm RBBE~\citet{pacchiano2020regret} does not match the lower bound.
    \item By comparing the upper and lower bounds derived by this work to existing adaptivity results, we \highlight{draw connections between} the statistical difficulty of adaptivity in three types of function spaces: RKHSs, Sobolev spaces, and H\"older spaces. 
\end{enumerate}
\highlight{A summary of our results amongst existing results can be found in Table~\ref{table: results summary}. Our main results (Section~\ref{sec: lower bound}) are stated for more general kernels but in Table~\ref{table: results summary} only results with Mat\'ern-$\nu$ (Definition~\ref{def: matern kernels}) kernels are shown as an example, for clear comparisons. For adaptive results, the values $\tilde{\nu}$ and $R$ are input parameters to the adaptive algorithms, such that they achieve (non-adaptive) minimax regret rates if the true parameter satisfies $\nu = \tilde{\nu}$ (for Mat\'ern RKHS) or $\alpha = R$ (for H\"older spaces). We use $\tilde{O}$ to denote the asymptotic regret rate of $T$. $\tilde{O}$ omits dependence on other parameters such as the radius of the RKHS ball $B$ (Section~\ref{sec: problem setting}), any constant factors, and $log$ factors of $T$ unless otherwise specified.} 
\begin{table*}[t]
\caption{Summary of Our Results and Comparison to Existing Results }
\centering
\label{table: results summary}
\begin{tabular}{|ll|l|l|}
\hline
\multicolumn{2}{|c|}{Regret} &
  RKHS of Matern-$\nu$: $\rkhs_{\kernel,\nu}(\domain)$ $=\sobolevspace^{\nu+\frac{d}{2}}(\domain)$ &
  H\"older Space: $\holderspace^{\alpha}(\domain)$ \\ \cline{3-4}
\multicolumn{2}{|l|}{} &
  \multicolumn{2}{c|}{Relationship: $\rkhs_{\kernel,\nu}(\domain) =\sobolevspace^{\nu+\frac{d}{2}}(\domain) \subset \holderspace^{\alpha=\nu}(\domain)$} \\ \hline
\multicolumn{2}{|l|}{Non-adaptive minimax} &
  \begin{tabular}[c]{@{}l@{}}$\tilde\Theta(T^{\frac{\nu+\paramdim}{2\nu+\paramdim}})$\\ \citet{valko2013finite},\citet{scarlett2017lower}\end{tabular} &
  \begin{tabular}[c]{@{}l@{}}
  $\tilde{\Theta}\left(T^{\frac{\paramdim+\alpha}{\paramdim + 2\alpha}}\right)$\\\citet{liu2021smooth},\citet{wang2018optimization}\end{tabular} \\ \hline
\multicolumn{1}{|l|}{\multirow{2}{*}{\begin{tabular}[c]{@{}l@{}}Adaptive  \\ ($\paramdim = 1$)\end{tabular}}} &
  Upper bound &
  \begin{tabular}[c]{@{}l@{}}
$\tilde{O}\left(T^\frac{1+2\tilde{\nu}+\tilde{\nu}\nu}{(1+2\tilde{\nu})(1+\nu)}\right)$, for $\nu<\tilde{\nu}$ \\ $\tilde{\nu}$: Input to adaptive algorithm. \\ This work (Theorem~\ref{thm: CORRAL with kernelised bandits})\end{tabular} &
  \begin{tabular}[c]{@{}l@{}}
  $\tilde{O}\left(T^\frac{1+2R+R\alpha}{(1+2R)(1+\alpha)}\right)$, for $\alpha<R$ \\$R$: Input to adaptive algorithm. \\
  \citet[Theorem 8]{liu2021smooth} \end{tabular}\\ \cline{2-4} 
\multicolumn{1}{|l|}{} &
  Lower bound &
 \begin{tabular}[c]{@{}l@{}}$\tilde{\Omega}\left(T^\frac{1^2+2\tilde{\nu}+\tilde{\nu}\nu}{(1+2\tilde{\nu})(1+\nu)}\right)$, for $\nu<\tilde{\nu}$ \\
 This work (Corollary~\ref{thm: lower bound with Matern kernel}) \end{tabular}&
  \begin{tabular}[c]{@{}l@{}}
 $\tilde{\Omega}\left(T^\frac{1^2+2R\alpha+R\alpha}{(1+2R)(1+\alpha)} \right)$, for $\alpha<R\leq 1$ \\ \citet[Theorem 3]{locatelli2018adaptivity}\end{tabular} \\ \hline
\end{tabular}
\end{table*}

\paragraph{Relationship with Neural Bandits.} The kernelised bandit formulation has implications for optimization of more complex functions under the bandit setting as well, such as neural network functions. The Neural Tangent Kernel (NTK) literature~\citep{jacot2018neural, arora2019exact,lee2019wide, bietti2020deep, chen2020deep} argue that over-parameterized neural networks can be approximated by functions in an RKHS of some composite kernel named the Neural Tangent Kernel, given that the network is sufficiently wide and the training is lazy~\citep{chizat2019lazy}.
Recent advances in this field establish interesting connections between the structure of a neural network and the regularity of its corresponding NTK. 
For example, \citet{vakili2021uniform} consider wide fully-connected neural networks with activation functions with smoothness $s$. The show that the RKHS of the NTK of such a network is norm equivalent to, or embedded in, the RKHS of a \matern-$\nu$ kernel with $\nu=s-\frac{1}{2}$. The value of $\nu$ dictates the differentiability of functions in the RKHS. Hence, the neural network functions considered in~\citet{vakili2021uniform} \highlight{are approximated by} functions in the RKHS of a \matern-$\nu$ kernel.\footnote{The result in Corollary 3 of~\citet{bietti2020deep} can be thought as a special case of when $s=1$, since the activation function considered is ReLU.} 
These connections imply that adaptivity to the kernel regularity {in kernelised bandits} can potentially 
be extended to adaptivity to the structure of neural networks (such as smoothness of the activation functions considered in~\citet{vakili2021uniform}) {in neural network bandits}.

The rest of the paper is structured as follows. In Section~\ref{sec: related work}, we discuss relevant prior works. In Section~\ref{sec: problem setting} we state the problem formulation. In Section~\ref{sec: main result} we present the main result of this paper, a lower bound on adaptivity to kernel regularity. In Secion~\ref{sec: adaptive upper bounds} we discuss upper bounds of existing adaptive algorithms and whether they match the lower bound. In Section~\ref{sec: connect with Holder adaptivity} we connect adaptivity to kernel regularity and adaptivity to H\"older exponents. Finally, we discuss the limitations and future directions of our work in Section~\ref{sec: discussion}.

\section{Related Work}\label{sec: related work}
\paragraph{Kernelised Bandit}
In kernelised bandit problems, the reward function lies in a reproducing kernel Hilbert space (RKHS). This problem has been studied by many previous works, under the assumption that the kernel and other parameters (such as the upper bound on the function's RKHS norm) are known. \citet{valko2013finite} take a frequentist approach and design a SuperKernelUCB algorithm, based on applying the kernel trick to the (Sup)LinREL and (Sup)LinUCB algorithms~\citep{auer2002using, chu2011contextual}. The same technique is later used in extension to neural networks by~\citet{kassraie2021neural}, who propose SupNTKUCB which works with neural networks. SupKernelUCB achieves $\tilde{O}(\sqrt{T\gamma_T})$ regret where $\gamma_T$ is the maximum information gain between $T$ total observations and the underlying function. For common kernels such as the Mat\'ern-$\nu$ kernels, this regret is minimax optimal in its dependence on $T$ (except for log factors), by the lower bound provided later in~\citet{scarlett2017lower}. However, SupKernelUCB relies on a batching technique that makes the algorithm performs poorly in practice~\citep{calandriello2019gaussian,janz2020bandit}.
In the (parallel) Bayesian setting (the Gaussian Process bandit problem), the underlying function is assumed to be drawn from a GP. GP-UCB algorithm~\citep{srinivas2009gaussian, chowdhury2017kernelized,janz2020bandit} achieves the same regret bound as SupKernelUCB $\tilde{O}(\sqrt{T\gamma_T})$ in the GP setting but becomes suboptimal (sometimes with linear regret rate) in the RKHS setting with a $\tilde{O}(\gamma_T\sqrt{T})$ regret~\citep{vakili2021open}. 

\paragraph{Adaptivity in Kernelised Bandit}
This problem we consider falls under the scope of model misspecification in bandit setting, which has been studied for linear functions and H\"older-smooth functions~\citep{du2019good,foster2019model, lattimore2020learning, zhu2021pareto, locatelli2018adaptivity, liu2021smooth}. For H\"older functions, in particular, \citet{locatelli2018adaptivity, hadiji2019polynomial} provide a lower bound indicating that it is impossible to achieve minimax adaptivity to the H\"older exponent. In this work, we convey a similar message with respect to the regularity of RKHS.  
For adaptivity in kernelised bandit problems, \citet{berkenkamp2019no} propose an algorithm with sublinear regret for when the lengthscale parameter (Definition~\ref{def: matern kernels}) and upper bound on the RKHS norm (equation~\ref{eq: rkhs ball}) are unknown. \citet{neiswanger2021uncertainty} develop robust confidence sequence under the Bayesian framework to use in adaptive methods for GP optimization when the prior mean and/or covariance parameters are unknown. They conduct simulations for optimization on functions drawn from GPs but do not provide explicit regret analyses. \citet{bogunovic2021misspecified} develop methods for $\epsilon$-misspecification, where the underlying function can be arbitrarily non-smooth, but is approximated by functions in a (known) RKHS with an $\epsilon$-error in infinity norm. They prove a $\Omega(\epsilon T)$ lower bound for this setting and derived a matching upper bound. 
However, note that between two function spaces, the approximation error is a constant value and does not depend on $T$. Since a constant $\epsilon$ means an inevitable linear regret ($\Omega(\epsilon T)$), the $\epsilon$-misspecification setting~\citep{bogunovic2021misspecified} does not directly apply to \highlight{adaptation to the kernel parameters}. In the Meta-learning regime, \citet{kassraie2022meta} consider RKHS with unknown kernels that are sparse combinations of known base kernels and proves that a Meta-learned kernel can yield sublinear regret. However, since the kernel is Meta-learned, it relies on offline tasks as training data. We do not assume the availability of offline data in the (fully online) bandit setting. 

To summarize, prior works (to the best of our knowledge) only consider parameters that influence the regret rate in \highlight{polynomial factors} while our focus is on the regularity parameter which affects the rate in the exponent of $T$. 
\paragraph{General Model Selection for Bandit}
Another line of recent works on model selection in bandit settings makes less stringent assumptions on the underlying function. These works consider algorithms based on a ``corralling'' mechanism, where a master algorithm ``corrals'' several base algorithms as arms and each base algorithm selects actions with different principles. The base algorithms usually assume different function spaces. \citet{agarwal2017corralling, pacchiano2020model} propose an algorithm named CORRAL where the master algorithm is based on online mirror descent. In certain cases, CORRAL performs comparably to the best base algorithm running standalone.\footnote{\citet{arora2021corralling} also study the problem of corralling bandit algorithms in the stochastic setting, but only finite-armed case is considered.} 
\citet{pacchiano2020regret} propose the Regret Bound Balancing and Elimination (RBBE) which uses a (simpler) stochastic master algorithm and an additional base-algorithm-elimination step. We refer readers to Section~\ref{sec: adaptive upper bounds} for details about these two methods and their performance in our problem setting. 

\section{Problem Setting}\label{sec: problem setting}
\paragraph{Problem Formulation}
Consider the problem of zeroth-order black-box optimization under bandit feedback. The learner interacts with a stochastic environment in a sequential manner. This problem is also formulated as stochastic continuum-armed bandit. At time step $t\in\{1,\dots, T\}$, the learner chooses an action $\action_t$ from the compact domain $\domain = [0,1]^\paramdim$, and receives a reward $\reward_t$. The reward is a noisy observation of the underlying reward function $\func: \domain \rightarrow \sR$:
\begin{equation}\label{eq: bandit problem}
    \reward_t = \func(\action_t) + \noise_t,
\end{equation}
where the noise variable $\noise_t$ follows a zero-mean \highlight{sub-Gaussian distribution} (see Theorem~\ref{thm: lower bound rkhs}). 
The optimization objective is the cumulative (pseudo-)regret defined as follows.  
\begin{equation}\label{eq: cumulative regret}
    R_T = \sum_{t=1}^T \func(\action^*) - \func(\action_t),
\end{equation}
where $\action^*$ is the global maximizer of $\func$, unknown to the learner. \highlight{Results in this paper are expressed in expected cumulative (pseudo-)regret $\E[R_T]$, where the expectation is taken over the randomness of $\{\action_t\}_{t=1\dots T}$.}

\paragraph{Kernelised Bandit}
We consider the setting where $\func$ is square-integrable and resides in an RKHS $\rkhs_\kernel$ of a symmetric, positive-definite kernel $\kernel: \sR^\paramdim \times \sR^\paramdim \rightarrow \sR$. The RKHS is unique given the kernel~\citep[Theorem 12.11]{wainwright2019high}. We denote the RKHS of $\kernel$ on domain $\domain$ as $\rkhs_\kernel(\domain)$. 
In this work, we restrict our attention to \emph{translation-invariant} kernels, precisely, kernels that satisfy the following: $\kernel(\action, \action') = \transinvkernel(x-x')$, for some function $\transinvkernel: \sR^\paramdim \rightarrow \sR$.
For a translation-invariant kernel, the regularity of functions in the RKHS is captured by the Fourier transform of the kernel. \highlight{Precisely, we have the following definition when the domain is $\sR^\paramdim$.} Let $\hat g(\omega)$ denote the Fourier transformation~\citep{wendland2004scattered, williams2006gaussian} of a function $g$ as $\forall \omega\in\sR^\paramdim$.
\begin{align}
    \rkhs_\kernel\highlight{(\sR^\paramdim)} = \{&\func\in \lspace_2(\sR^\paramdim) \cap C(\sR^\paramdim): \label{eq: definition rkhs in Fourier}\\
    &\Vert \func\Vert_{\rkhs_\kernel}: = (2\pi)^{-\paramdim/2}\int_{\sR^\paramdim}\frac{\vert\hat{\func}(\omega)\vert^2}{\hat{\kappa}(\omega)} d\omega<\infty\}.\nonumber
\end{align}
\highlight{When the domain $\domain$ is a subset of $\sR^\paramdim$, $\hat{\kappa}$ still captures the regularity of $\rkhs_{\kernel}(\domain)$, via a norm equivalency result that holds as long as $\domain$ has a Lipschitz boundary. Details can be found in Section~\ref{sec: norm equivalency}, Lemma~\ref{lemma: norm equivalency between rkhs and Sobolev}.}
We write $\Vert f\Vert_{\kernel} \stackrel{\triangle}{=} \Vert f\Vert_{\rkhs_\kernel(\domain)}$ for simplicity.  
We apply the common assumption~\citep{srinivas2009gaussian,valko2013finite} that the RKHS norm of $\func$ is upper bounded by a value $B, 0<B<\infty$:   
\begin{equation}\label{eq: rkhs ball}
    \func\in \rkhs_{\kernel}(\domain, B):= \{\func: \func\in\rkhs_{\kernel}, \Vert \func \Vert_\kernel\leq B\}.
\end{equation}
\highlight{We refer to $\rkhs_{\kernel}(\domain, B)$ as a ball in the RKHS with radius~$B$. } 

\section{Main Result: Adaptivity Lower Bound}\label{sec: main result}
In this section, we present the main result, a lower bound on adaptivity to the regularity of kernel (Theorem~\ref{thm: lower bound rkhs}).
The regularity of a translation-invariant kernel is expressed as the decay rate of its Fourier transformation (\eqref{eq: definition kernel Fourier decay}). 
We next \highlight{instantiate this idea with a norm equivalency result between an RKHS and a Sobolev space. The norm equivalency result is dependent on the kernel Fourier decay rate}. The proof of Theorem~\ref{thm: lower bound rkhs}, in turn, relies on this norm equivalency as well. 
\subsection{Norm Equivalency Between RKHS and Sobolev Space}\label{sec: norm equivalency}
Consider integer-order Sobolev space $\sobolevspace^{\sobolevint,p}(\domain)$ where $\sobolevint, p$ are integers greater or equal to $1$. 
We define the following notions for a multi-index vector $\multiidx=(\multiidx_1\dots \multiidx_\paramdim)$: $\vert \multiidx\vert {=} \multiidx_1 + \dots + \multiidx_\paramdim$, $\multiidx! = \multiidx_1!\dots \multiidx_\paramdim!$ and $\action^\multiidx = \action_1^{\alpha_1}\dots \action_\paramdim^{\multiidx_\paramdim}$. Let $D^{(\multiidx)} = \frac{\partial^{\vert \multiidx\vert}}{\partial \action_1^{\multiidx_1}\dots \partial \action_\paramdim^{\multiidx_\paramdim}}$ denote the \highlight{multivariate mixed partial weak derivative}. The Sobolev space and corresponding Sobolev norm ($\Vert \cdot \Vert_{m,p,\domain}$) are defined as follows.
\begin{align}
    &\sobolevspace^{\sobolevint,p}(\domain) = \{\func\in \lspace_p(\domain): D^{(\multiidx)}f\in \lspace_p(\sR^d), \forall \vert \multiidx \vert\leq m\}, \label{eq: definition Sobolev space}\\
    &\qquad \Vert \func \Vert_{m,p,\domain} := \left(\sum_{\vert \multiidx \vert\leq \sobolevint}\int \vert D^{(\multiidx)}\func(x)\vert^p dx\right)^\frac{1}{p}. \label{eq: definition Sobolev norm}
\end{align}
We refer to $\sobolevint$ as the order of the Sobolev space.   
Furthermore, define the $j$-th order  seminorm~\citep[Definition 4.11]{adams2003sobolev} $\vert{\cdot}\vert_{j,p,{\domain}}$ with integer $j\leq m$, 
which is the sum of $\lspace_p$ norm of its $j$-th weak derivatives.
\begin{equation}\label{eq: def seminorm Sobolev}
    \vert \func \vert_{j,p,\domain} = \left(\sum_{\vert \multiidx\vert=j}\int\vert D^{(\multiidx)}\func(x)\vert^p dx\right)^\frac{1}{p}.
\end{equation}
\highlight{In correspondence to the RKHS ball (equation~\ref{eq: rkhs ball})}, we define a Sobolev ball with radius $\sobolevradius$ as the set of functions whose $\sobolevint$-th order \emph{seminorm} are upper bounded by $\sobolevradius$.
\begin{equation}\label{eq: definition Sobolev ball}
    \sobolevspace^{\sobolevint,p}(\domain, \sobolevradius) = \{\func\in \sobolevspace^{\sobolevint,p}(\domain): \vert \func \vert_{\sobolevint,p,\domain} \leq \sobolevradius\}.
\end{equation}
When $p=2$, the Sobolev space is equivalent to the RKHS of a translation-invariant kernel $k$. This connection plays an important role in the analysis. We consider only Sobolev spaces with $p=2$, and hence abbreviate $\sobolevspace^{\sobolevint}(\domain) \stackrel{\triangle}{=} \sobolevspace^{\sobolevint, 2}(\domain)$. The precise norm equivalency is introduced in the following lemma. 
\begin{lemma}\citet[Corollary 10.48]{wendland2004scattered}\label{lemma: norm equivalency between rkhs and Sobolev}
    Let $\kernel: \sR^\paramdim \times \sR^\paramdim \rightarrow \sR$ be a translation-invariant kernel function such that $\kernel(\cdot, \cdot) = \transinvkernel(\cdot-\cdot)$ for $\transinvkernel \in \lspace_1(\sR^\paramdim)$. \highlight{Suppose $\Omega\in \sR^\paramdim$ is a domain with Lipschitz boundary.} Suppose $\hat{\kappa}$ has the following polynomial decay rate of $\fourierrate$, for $\fourierrate>\paramdim/2, \fourierrate\in\mathbb{N}$,
    \begin{align}
        &c_1(1+\Vert\omega\Vert^2_2)^{-\fourierrate}\leq \hat{\kernel}(\omega)\leq c_2(1+\Vert\omega\Vert^2_2)^{-\fourierrate}, \forall w\in\sR^\paramdim,\label{eq: definition kernel Fourier decay}
    \end{align}
    for some constants $0<c_1\leq c_2$. 
    Then, the associated RKHS $\rkhs_\kernel(\Omega)$ is norm equivalent to the Sobolev space $W^{\sobolevint}(\Omega)$ with $\sobolevint = \fourierrate$. 
\end{lemma}
Having established the equivalency between RKHS and Sobolev spaces, we further introduce some notions to quantify the relationship between Sobolev seminorm (which is the radius of Sobolev balls) and RKHS norm in the following lemma. 
\begin{lemma}\label{lemma: norm equivalency between semi Sobolev norm and RKHS norm}
Suppose that $\sobolevint$ is a positive integer larger than $\paramdim/2$. Let $\Omega$ be a finite-width domain with Lipschitz boundary. \highlight{Let $\sobolevspace_0^{\sobolevint, p}(\Omega)$ denote the closure of $C^\infty_0(\Omega)$ (set of functions that have compact support in $\Omega$ and, together with their infinite order of partial derivatives, are continuous) in $\sobolevspace^{\sobolevint,p}(\Omega)$~\citet{adams2003sobolev}. }
Then, the $m$-th Sobolev seminorm of $\func$ can be bounded by its RKHS norm with respect to a translation-invariant kernel $\kernel$ with Fourier decay rate $\sobolevint$. Precisely, 
\begin{equation}\label{eq: norm equivalency between RKHS norm and Sobolev semi-norm}
    \normequivl \vert \func \vert_{\sobolevint, 2} \leq \Vert \func \Vert_{\rkhs_\kernel} \leq \normequivu \vert \func \vert_{\sobolevint, 2},
\end{equation}
for some constants $0<\normequivl<\normequivu$.
\end{lemma}
The constants $\normequivl, \normequivu$ are used globally in this work and appear in the lower bound in Section~\ref{sec: lower bound}. The proof of Lemma~\ref{lemma: norm equivalency between semi Sobolev norm and RKHS norm} can be found in Appendix~\ref{sec: proof of norm equivalence between RKHS and Sobolev seminorm}.

\subsection{Lower Bound on Adaptivity to Kernel Regularity}\label{sec: lower bound}
Theorem~\ref{thm: lower bound rkhs} presents our lower bound for adapting between a pair of RKHSs of different (kernel) regularities. An intuitive interpretation of the theorem is as follows. Consider a nested pair of balls in two RKHSs. Suppose both kernels satisfy the conditions in Lemma~\ref{lemma: norm equivalency between rkhs and Sobolev} but with different Fourier decay rates: $\sobolevint_1 \in \sN$ and $\sobolevint_2 \in \sN$ such that $0<\sobolevint_1 < \sobolevint_2$. If an algorithm that is oblivious to the true regularity value somehow achieves a small (for example, minimax optimal) regret on all functions inside the (smoother) RKHS ball with parameter $\sobolevint_2$, this algorithm will suffer a price of larger (suboptimal) regret on at least one function inside the (rougher) RKSH ball with parameter $\sobolevint_1$. 
For the lower bound analysis, we consider $\paramdim=1$ and leave the extension of the lower bound to $\paramdim>1$ as a future direction (Section~\ref{sec: discussion}). 
\begin{theorem}\label{thm: lower bound rkhs}
    Consider the problem setting in Section~\ref{sec: problem setting} with noises $\{\noise_t\}_{t=1\dots T}$ that are $\frac{1}{4}$-subgaussian. Let $\regupperbound$ be a positive number, let $\sobolevint_1, \sobolevint_2$ be two positive integers that satisfy $\sobolevint_1<\sobolevint_2$. There exist two positive values $B_1$ and $B_2$, such that the following statement is true.
    Consider an algorithm that achieves in the RKHS of a kernel $\kernelfourier{\sobolevint_2}$ with Fourier decay rate $\sobolevint_2$ the following regret upper bound. 
    \begin{equation}
    \sup_{\func\in\rkhs_{\kernelfourier{\sobolevint_2}}(\domain, B_2)}\E[R_T] \leq \regupperbound. 
    \end{equation}
    Then, the regret of this algorithm in a (less smooth) RKHS of another kernel $\kernelfourier{\sobolevint_1}$ with Fourier decay rate $\sobolevint_1$ is lower bounded by the following. \highlight{Suppose that functions in the function spaces have bounded $\lspace_2$ norm.}\footnote{Functions in Sobolev spaces and RKHSs are square-integrable.}
    \begin{align}
        &\sup_{\func\in \rkhs_{\kernelfourier{\sobolevint_1}}(\domain, B_1)} \E[R_T]\geq 
        \\
        &\qquad \qquad \frac{1}{8} \left(\frac{C(\sobolevint_1)}{32}\right)^{\frac{\sobolevint_1-1/2}{\sobolevint_1+1/2}}~ {\left(\frac{B_1}{\normequivu}\right)}^{\frac{1}{\sobolevint_1+1/2}}~\regupperbound^{-\frac{\sobolevint_1-1/2}{\sobolevint_1+1/2}}~T.\nonumber
    \end{align}
    {Here, $C(m_1)$ denotes a constant that depends on $m_1$.}
\end{theorem}
\highlight{It is worth noting that, although the lower bound has a factor of $T$, the regret is not necessarily linear in $T$, because $\tilde{R}$ also depends on $T$ and in fact usually ranges from $\tilde{O}(\sqrt{T})$ to $\tilde{O}(T)$.} The full version of this theorem is presented as Theorem~\ref{thm: lower bound rkhs, full version} in Appendix~\ref{sec: full lower bound rkhs}, where we state the full constraints on the radius values $B_1$ and $B_2$. Since $B_1$ and $B_2$ are only upper bounds on the RKHS norm and \emph{not} the kernel regularity that we focus on, we present only the concise version here to show the adaptivity difficulty with respect to regularity parameters $\sobolevint_1$ and $\sobolevint_2$.

\subsubsection{Proof Sketch}\label{sec: proof sketch lower bound}
The proof of Theorem~\ref{thm: lower bound rkhs} consists of two key parts. The first part is constructing the hypothesis functions, in which we borrow ideas from lower bounds in regression problems~\citep{tsybakov2004introduction}. The second part is lower bounding the cumulative regret, given the constructed hypothesis functions, where we follow~\citet[Section 2.2]{hadiji2019polynomial}. Intuitively, the second part shows that if any player achieves a small regret on all the smoother functions, then it inevitably incurs large regret on the rougher functions in the space, because of its disproportionally small amount of exploration. 
The method in~\citet{hadiji2019polynomial} is itself an improved version of the adaptivity lower bound for H\"older spaces proposed in~\citet{locatelli2018adaptivity}. 

\subsubsection{A Sobolev Version of the Lower Bound}\label{sec: sobolev version lower bound}
It is convenient to construct functions with compact support and finite Sobolev semi-norms from an infinitely-differentiable base function, such as the bump function~\citep{tsybakov2004introduction}. 
On the other hand, directly constructing functions with finite RKHS norms~\citep[Section III.A]{scarlett2017lower} involves inverse Fourier transformation of the bump function and thus leads to wavelet-like functions with non-compact support. 
Therefore, it is more natural for us to first consider functions in (integer-order) Sobolev spaces as hypothesis functions, and then use the norm equivalency result between Sobolev spaces and RKHSs to prove the lower bound. More precisely, the hypothesis functions constructed in the proof reside in a Sobolev ball $\sobolevspace^{\sobolevint}(\domain, L)$, for some (integer) order $m$ and radius $L$. Via the norm equivalency (Lemma~\ref{lemma: norm equivalency between semi Sobolev norm and RKHS norm}), those functions also resides in a RKHS ball of a kernel with Fourier decay rate $\sobolevint$. 

As a result, there is a Sobolev version of the adaptivity lower bound. Informally, let $\sobolevint_1, \sobolevint_2$ be two positive integers such that $\sobolevint_2>\sobolevint_1$. If an algorithm achieves a $\tilde{R}$ regret upper bound in the smoother Sobolev space $\sobolevspace^{\sobolevint_2}(\domain)$, then its regret over functions in $\sobolevspace^{\sobolevint_1}(\domain)$ is lower bounded by $\Omega(\regupperbound^{-\frac{\sobolevint_1-1/2}{\sobolevint_1+1/2}}~T)$. We formally state the Sobolev version of the adaptivity lower bound in Theorem~\ref{thm: lower bound Sobolev} in Appendix~\ref{sec: proof of lower bound rkhs}. 
The two lower bounds share the same proof structure, connected via the norm equivalency in Lemma~\ref{lemma: norm equivalency between rkhs and Sobolev}.

\subsubsection{Impossibility Result for Mat\'ern Kernels}\label{sec: impossibility result for Matern}
For the Mat\'ern-$\nu$ family of kernels~\citep{matern1960spatial}, an implication of Theorem~\ref{thm: lower bound rkhs} is that no algorithm can achieve minimax adaptivity between two RKHSs if they have different regularity. 
Therefore, we also refer to this lower bound as an impossibility result for adaptivity to the kernel regularity. We formally define Mat\'ern-$\nu$ family of kernels in Definition~\ref{def: matern kernels}. 
\begin{definition}\label{def: matern kernels}
    The Mat\'ern-$\nu$ kernel and its Fourier transformation are defined as follows for dimension $\paramdim$.
    \begin{align}
        &\kernel_{\text{Mat\'ern},\nu}(\action, \action') \\
        &\quad = \frac{2^{1-\nu}}{\Gamma(\nu)} \left(\frac{\sqrt{2\nu}\Vert x-x'\Vert_2}{l}\right)^\nu J_\nu(\frac{\sqrt{2\nu}\Vert x-x'\Vert_2}{l}), \label{eq: Matern kernel}\\
        &\hat{\kernel}_{\text{Mat\'ern},\nu}(\omega) = c_1(\frac{2\nu}{l^2}+\Vert\omega\Vert_2^2)^{-(\nu+\frac{\paramdim}{2})}.\label{eq: Matern kernel Fourier}
    \end{align}
    where $c_1 =\frac{2^d\pi^{d/2}\Gamma(\nu+d/2)(2\nu)^\nu}{\Gamma(\nu)l^{2\nu}}$, $J_\nu$ is the modified Bessel function of the second kind, $l$ is the length-scale, and $\nu>0$ is the regularity parameter. In this work, we assume for simplicity that the length-scale is set \highlight{to $\propto \sqrt{2\nu}$}.
\end{definition}
The Fourier transformation of a Mat\'ern kernel with regularity parameter $\nu$ decays with a rate of $\nu+\frac{d}{2}$ (equation~\eqref{eq: Matern kernel Fourier}). 
Therefore, we can instantiate the impossibility result for Mat\'ern kernels. The result is presented in Corollary~\ref{thm: lower bound with Matern kernel}. Precisely, for $0<\nu_1<\nu_2$, if an adaptive algorithm achieves minimax regret rate on a Mat\'ern RKHS with regularity $\nu_2$, then it has a strictly suboptimal regret rate on the RKHS with $\nu_1$. 
\begin{corollary}\label{thm: lower bound with Matern kernel}
   Suppose the problem is the same as defined in Theorem~\ref{thm: lower bound rkhs}. Let $\nu_1, \nu_2$ be real numbers that satisfy $0<\nu_1< \nu_2$ and $\nu_1+\frac{1}{2}\in \sN, \nu_2+\frac{1}{2}\in \sN$. There exist two positive values $B_1, B_2$, such that the following statement is true.
   Suppose an algorithm oblivious to the true regularity parameter value achieves the following minimax optimal regret \footnote{Omitting the dependence on the upper bound on RKHS norm.} on $\rkhs_{\kernel_{\text{Mat\'ern},\nu_2}}(\domain, B_2)$,
    \begin{equation}
        \sup_{\func\in\rkhs_{\kernel_{\text{Mat\'ern},\nu_2}}(\domain, B_2)} \E[R_T] = \tilde{O}\left(T^{\frac{\nu_2+1}{2\nu_2+1}}\right), 
    \end{equation}
    then the regret of this algorithm on RKHS $\rkhs_{\kernel_{\text{Mat\'ern},\nu_1}}(\domain, B_1)$ is lower bounded by the following.
    \begin{equation}\label{eq: adaptivity lower bound on matern rkhs}
        \sup_{\func\in\rkhs_{\kernel_{\text{Mat\'ern},\nu_1}}(\domain, B_1)} \E[R_T] = \tilde{\Omega}\left(T^{\frac{\nu_1\nu_2+2\nu_2+1}{(\nu_1+1)(2\nu_2+1)}}\right), 
    \end{equation}
\end{corollary}
The proof of Corollary~\ref{thm: lower bound with Matern kernel} is an application of Theorem~\ref{thm: lower bound rkhs} and can be found in Appendix~\ref{sec: proof of corollary of lower bound with Matern kernels}. The cumulative regret rate in~\ref{eq: adaptivity lower bound on matern rkhs} is suboptimal compared to the minimax rate which is $\tilde{O}(T^{\frac{\nu_1+1}{2\nu_1+1}})$ (see Section~\ref{sec: non-adaptive minimax algos} for non-adaptive minimax rates). Therefore, Theorem~\ref{thm: lower bound rkhs} is an impossibility result for adaptivity to kernel regularity with Mat\'ern kernels.

\section{Upper Bounds of Adaptive Algorithms}\label{sec: adaptive upper bounds}
We consider two adaptive algorithms particularly: CORRAL from~\citet{agarwal2017corralling, pacchiano2020model} and Regret Bound Balancing and Elimination (RBBE) from~\citep{pacchiano2020regret}.  The two algorithms (i) can be applied to the problem of adaptation to kernel regularity and (ii) have explicit regret guarantees in this setting.

The adaptive algorithms, however, need base algorithms that are non-adaptive minimax optimal. We first provide an overview of such non-adaptive algorithms for kernelised bandit in Section~\ref{sec: non-adaptive minimax algos}. Then, we derive adaptivity upper bounds of CORRAL and RBBE in Section~\ref{sec: corral} and Section~\ref{sec: rbbe} respectively. 
For concreteness, we only consider RKHS of Mat\'ern-$\nu$ kernel (Definition~\ref{def: matern kernels}) in this section. To match the lower bound, we set $\paramdim=1$. Comparison of the upper bounds to the lower bound (Theorem~\ref{thm: lower bound rkhs}), shows that CORRAL (coupled with minimax optimal base algorithms) can match the lower bound in dependence on $T$ between certain pairs of values for $\nu$. 

\subsection{Overview: Non-adaptive Minimax Algorithms}\label{sec: non-adaptive minimax algos}
We discuss the theoretical performance of algorithms developed for kernelised bandits in Section~\ref{sec: kernelucb and gpucb}. We show that a recent algorithm that is designed for continuum-armed bandit in H\"older spaces~\citep{liu2021smooth} is also optimal over functions in RKHS of Mat\'ern kernels in Section~\ref{sec: ucb-meta}. 

\subsubsection{SupKernelUCB and GP-UCB for RKHS}\label{sec: kernelucb and gpucb}
Recall that the lower bound (in terms of $T$) on cumulative regret for kernelised bandit with Mat\'ern-$\nu$ kernels $\kernel_{\text{Mat\'ern},\nu}$ is $\Omega(T^\frac{\nu+1}{2\nu+1})$, as proved by~\citet{scarlett2017lower}. There are mainly two types of algorithms applicable for the kernelised bandit problem: (i) GP-UCB~\citep{srinivas2009gaussian} and its variants~\citep{chowdhury2017kernelized,janz2020bandit}, and (ii) KernelUCB and its \textit{Sup}-variant~\citet{valko2013finite}. The GP-UCB-style algorithms display a non-trivial empirical advantage over the impractical SupKernelUCB. 
That being said, GP-UCB is suboptimal theoretical upper bounds for certain types of kernels under the RKHS assumption, including for Mat\'ern-$\nu$ kernels. In the RKHS of a Mat\'ern kernel $\kernel_{\text{Mat\'ern},\nu}$, GP-UCB achieves a regret of $\tilde{O}(T^{\frac{\nu+\frac{3}{2}}{2\nu+1}})$.\footnote{The suboptimality of GP-UCB is discussed more extensively in~\citet{vakili2021open}}. On the other hand, SupKernelUCB matches the lower bound with a regret rate of $\tilde{O}(T^{\frac{\nu+1}{2\nu+1}})$.\footnote{The analysis of SupKernelUCB was originally for finite-armed setting, but~\citet[Appendix A.4]{cai2021lower} state that it can be extended to the continuum-armed setting where $\domain=[0,1]^\paramdim$, suffering only a $O(\paramdim \log(T))$ term in the regret.} 

\subsubsection{UCB-Meta for H\"older Space}\label{sec: ucb-meta}
Apart from the kernelised bandit algorithms discussed above, \citet{liu2021smooth} propose an algorithm for continuum-armed bandits in H\"older space with exponent $\alpha>1$ with regret upper bound that matches existing lower bounds~\citep{wang2018optimization,singh2021continuum} except log factors. This algorithm is named UCB-Meta. We show in Theorem~\ref{thm: UCB-Meta regret in RKHS} that UCB-Meta is naturally minimax optimal in dependence on $T$ over the RKHS of certain kernels.
\begin{theorem}\label{thm: UCB-Meta regret in RKHS} 
Consider the kernelised bandit problem where $\func\in \rkhs_{\kernel_{\text{Mat\'ern},\nu}}(\domain, B)$, where $\nu>0$ and $\nu+\frac{1}{2}\in\sN$. 
Then, UCB-Meta achieves the following regret upper bound, 
\begin{equation}\label{eq: UCB-Meta regret in RKHS}
    \sup_{\func\in\rkhs_{\kernel_{\text{Mat\'ern},\nu}}(\domain, B)} \E[R_T] = \tilde{O}\left(T^{\frac{\nu+1}{2\nu+1}}\right),
\end{equation} 
where $\tilde{O}$ omits dependence on radius of the RKHS ball $B$, constant factors depending on $\nu$, and $log$ factors of $T$.
\end{theorem}
The regret rate shown in Theorem~\ref{thm: UCB-Meta regret in RKHS} is derived from the result that $\rkhs_{\kernel_{\text{Mat\'ern},\nu}}(\domain)$ is embedded in a H\"older space $\holderspace^\alpha(\domain)$ with $\alpha = \nu$. The proof can be found in Appendix~\ref{sec: proof of UCB-Meta regret in RKHS}. \citet{singh2021continuum} have shown a similar argument while focusing mainly on the connection between Besov and H\"older spaces. 

\subsection{CORRAL as Adaptive Algorithm}\label{sec: corral}
The original CORRAL algorithm for model selection in the bandit setting is first proposed by~\citet{agarwal2017corralling}. The original CORRAL requires that modifications be made to each base algorithm for them to satisfy a stability condition (Definition 3 in~\citet{agarwal2017corralling}). These modifications, however, have to be made on a case-by-case basis.
Therefore, we use the smoothed version of CORRAL which is proposed by~\citet{pacchiano2020model}. The smoothed CORRAL puts a smoothing operation between the master algorithm and base algorithms and thus does not require modifications be made to the base algorithms. Smoothed CORRAL operates only with stochastic environments, which is satisfied by our assumptions (Section~\ref{sec: problem setting}). For simplicity, we refer to the smoothed version of CORRAL as CORRAL. CORRAL uses an adversarial online mirror descent algorithm as the master algorithm. 

Recall that a non-adaptive minimax kernelised bandit algorithm achieves $\tilde{O}(T^{\frac{\nu+1}{2\nu+1}})$ regret (Section~\ref{sec: non-adaptive minimax algos}), if instantiated with the correct parameter $\nu$. 
By plugging in the regret of base kernelised bandit algorithms in the general result in Theorem 5.3 in~\citet{pacchiano2020model}, we derive a adaptive upper bound for CORRAL in Theorem~\ref{thm: CORRAL with kernelised bandits}. 
CORRAL achieves sublinear $\tilde{o}(T)$ regret on all possible values of $\nu^*$ (See Theorem~\ref{thm: CORRAL with kernelised bandits}). Oppositely, a non-adaptive algorithm instantiated with parameter value $\tilde{\nu}$ does not have sublinear regret guarantees if the true parameter $\nu^*<\tilde{\nu}$, because the underlying function space $\rkhs_{\kernel_{\text{Mat\'ern},\nu^*}}$ is not contained in algorithm's hypothesis space. In Theorem~\ref{thm: CORRAL with kernelised bandits}, $\tilde{\nu}\in\vu$ is a parameter that is specified by the user and can be interpreted as the parameter that specifies the space on which the algorithm is configured to achieve minimax regret.
\begin{theorem}\label{thm: CORRAL with kernelised bandits}
    Consider the kernelised bandit problem where $\func\in \rkhs_{\kernel_{\text{Mat\'ern},\nu^*}}(\domain, B^*)$, $\nu^*+\frac{1}{2}\in \sN$ and $\nu^*, B^*$ unknown to the learner. Let $\vu=\{(\nu_1, B_1), (\nu_2,B_2), \dots, (\nu_M, B_M)\}$ be a list of candidate input value pairs such that $\vu$ specifies a nested set of RKHS: $\rkhs_{\kernel_{\text{Mat\'ern},\nu_1}}(\domain, B_1) \subset \rkhs_{\kernel_{\text{Mat\'ern},\nu_2}}(\domain, B_2) \subset \dots \rkhs_{\kernel_{\text{Mat\'ern},\nu_M}}(\domain, B_M)$. Suppose that $(\nu^*, B^*)\in\vu$. Let $\sA=\{\gA_i, i\in [M]\}$ be a set of (non-adaptive) minimax optimal kernelised bandit algorithms \highlight{with anytime regret guarantees}, each instantiated with the regularity and radius $(\nu_i, B_i)\in\vu$. The regret from running CORRAL with \highlight{input total time steps $T$ and} learning rate $\eta=\tilde{O}(T^{-\frac{1+\tilde{\nu}}{1+2\tilde{\nu}}})$ applied with base algorithms from $\sA$ is as follows.\footnote{$\tilde{O}$ omits dependence on radius of the RKHS ball $B$, constant factors depending on $\nu$, and $log$ factors of $T$.}
    
    \begin{equation}\label{eq: corral upper bound}
        \sup_{\func\in\rkhs_{\kernel_{\text{Mat\'ern},\nu^*}}}\E[R_T] = \tilde{O}\left(T^{\max(\frac{1+\tilde{\nu}}{1+2\tilde{\nu}}, \frac{1^2+2\tilde{\nu}+\tilde{\nu}\nu^*}{(1+2\tilde{\nu})(1+\nu^*)})}\right).
    \end{equation}
\end{theorem}
The proof of Theorem~\ref{thm: CORRAL with kernelised bandits} can be found in Appendix~\ref{sec: proof of thm CORRAL with kernelised bandits}. 
This result indicates that CORRAL achieves (i) minimax optimal rate $\tilde{O}\left(T^\frac{1+\nu^*}{1+2\nu^*}\right)$ in terms of $T$, if the underlying kernel regularity $\nu^*= \tilde{\nu}$; (ii) suboptimal rate $\tilde{O}\left(T^\frac{1+\tilde{\nu}}{1+2\tilde{\nu}}\right)$ if $\nu^*>\tilde{\nu}$ and (iii) suboptimal rate $\tilde{O}\left(T^\frac{1+2\tilde{\nu}+\tilde{\nu}\nu^*}{(1+2\tilde{\nu})(1+\nu^*)}\right)$ when $\nu^* < \tilde{\nu}$.
Let $\nu_1^*, \nu_2^*$ satisfying $\nu_1^* < \nu_2^*$ be two possible values of the true regularity that both satisfy the assumptions in~Theorem~\ref{thm: CORRAL with kernelised bandits}. Suppose $\nu_1^*\leq \tilde{\nu}<\nu_2*$. 
By Theorem~\ref{thm: CORRAL with kernelised bandits}, CORRAL achieves regret $\tilde{O}\left(T^\frac{1^2+2\tilde{\nu}+\tilde{\nu}\nu_1^*}{(1+2\tilde{\nu})(1+\nu_1^*)}\right)$ if the true parameter is $\nu_1^*$ and $\tilde{O}\left(T^\frac{1+\tilde{\nu}}{1+2\tilde{\nu}}\right)$ if the true parameter is $\nu_2^*$. 
By Theorem~\ref{thm: lower bound rkhs}, the lower bound over the rougher RKHS with $\nu_1^*$ is $\tilde{\Omega}\left(T^\frac{1+2\tilde{\nu}+\tilde{\nu}\nu_1^*}{(1+2\tilde{\nu})(1+\nu_1^*)}\right)$. The lower bound is matched by the upper bound in the exponent of $T$. 

In conclusion, CORRAL matches the adaptivity lower bound in the dependence on $T$ except log factors, between any pair of regularity values $(\nu_1^*, \nu_2^*)$, such that $\nu_1^*\geq \tilde{\nu},\nu_1^*+\frac{1}{2}\in \sN$ and $\nu_2^*<\tilde{\nu}, \nu_2^*+\frac{1}{2}\in \sN$. 

\highlight{Finally, note that in this subsection, the assumption is that the true parameter(s) are contained in the candidate set $\vu$. Hence, Theorem~\ref{thm: CORRAL with kernelised bandits} reflects the cost of adaptation (model selection), which is the difficulty of selecting the best base learner out of all candidates. If, however, the true parameter is not contained in $\vu$, then adaptive algorithms will incur another type of cost, namely the cost of “discretization”. This cost is generated from the difference between the true parameter and the closest value in $\vu$. Using an exponential~\citep{pacchiano2020model} or linear~\citep{liu2021smooth} grid for $\vu$ can usually incur a small cost of "discretization". }

\subsection{RBBE as Adaptive Algorithm}\label{sec: rbbe}
The regret bound balancing and elimination (RBBE) algorithm proposed in~\citet{pacchiano2020regret} achieves near-optimal regret in several adaptivity problems with linear function spaces. RBBE can be thought of as using a stochastic master algorithm that selects the base algorithm with the smallest candidate cumulative regret at each time. 
Therefore, it enjoys advantages such as gap-dependent regret bounds and high probability regret bounds.
Unlike CORRAL, it does not need a user-specified parameter to control the space over which the algorithm will achieve minimax optimal regret on. Instead, the algorithm achieves simultaneously on all possible values of $\nu^*$ the regret upper bound of $\tilde{O}(T^{\frac{1+4\nu^*+2{\nu^*}^2}{1+4\nu^* + 4{\nu^*}^2}})$. If we plug this upper bound in Theorem~\ref{thm: lower bound rkhs} for $\nu^* = \nu_2^*$, then a lower bound of $\tilde{\Omega}(T^\frac{(2
\nu_2^*+1)^2+2\nu_1^*{\nu_2^*}^2}{(2{\nu^*_2}^2+1)^2(\nu_1^*+1)})$ is incurred for when $\nu^* = \nu_1^*$, given that $0<\nu_1^*<\nu_2^*$. 
The upper bound of RBBE is larger than the lower bound in the exponent of $T$. A more detailed description of the RBBE algorithm and a formal statement of its adaptivity upper bound can be found in Appendix~\ref{sec: appendix rbbe details} and Theorem~\ref{thm: RBBE with kernelised bandit} therein. 

To summarize, although both CORRAL and RBBE as adaptive algorithms can achieve sublinear regret simultaneously on different kernel regularity, CORRAL has a better theoretical adaptivity in this problem. While RBBE fails to match the lower bound, CORRAL achieves the adaptivity lower bound for certain pairs of $\nu$ values for Mat\'ern-$\nu$ kernels. \footnote{It is our conjecture that the stochastic master used by RBBE (as opposed to the adversarial one in CORRAL) limits its model selection ability in certain cases.}

\section{Connection with Adaptivity to H\"older Exponents}\label{sec: connect with Holder adaptivity}
The adaptivity lower bound in Theorem~\ref{thm: lower bound rkhs} specifies the difficulty of adapting between two RKHSs of kernels with polynomial Fourier decay rate $\sobolevint_1$ and $\sobolevint_2$, where $0<\sobolevint_1< \sobolevint_2, \sobolevint_1\in \sN, \sobolevint_2\in\sN$. Recall that $\tilde{R}$ is the regret upper bound on the smoother RKHS with parameter $\sobolevint_2$. 
The lower bound on the RKHS specified by $\sobolevint_1$ depends inversely on $\tilde{R}$ through an $\Omega(T\cdot \tilde{R}^{-\frac{\sobolevint_1-1/2}{\sobolevint_1+1/2}})$ dependence. 

Shifting the perspective from RKHS to H\"older spaces, the adaptivity difficulty has been studied by~\citet{locatelli2018adaptivity, hadiji2019polynomial}, for a subset of values for the H\"older exponent $\alpha$. Precisely, Theorem 3 in~\citet{locatelli2018adaptivity} provides an $\Omega(T\cdot \tilde{R}^{-\frac{\alpha_1}{\alpha_1+1}})$ dependence as the lower bound, for adapting between two H\"older spaces with exponents $\alpha_1, \alpha_2$ satisfying $\alpha_1<\alpha_2\leq 1$.\footnote{Proving the adaptivity rate for when the exponents are larger than $1$ remains an open problem.} Here, $\tilde{R}$ is the upper regret bound on the smoother H\"older space $\holderspace^{\alpha_2}(\domain)$. 
We know by Lemma~\ref{lemma: norm equivalency between rkhs and Sobolev} that an RKHS $\rkhs_{\kernelfourier{\sobolevint_1}}(\domain)$ with kernel Fourier decay rate $\sobolevint_1$ is norm equivalent to Sobolev space $\sobolevspace^{\sobolevint_1}(\domain)$. Coupled with the Sobolev embedding theorem for integer-order Sobolev spaces~\citep[Theorem 5.4]{adams2003sobolev}, it is straightforward to see that $\rkhs_{\kernelfourier{\sobolevint_1}}(\domain)\subset \holderspace^\alpha(\domain)$, where $\alpha = \sobolevint_1 - \frac{1}{2}$ (Appendix~\ref{sec: proof of UCB-Meta regret in RKHS}). 

Note that we have the following equivalence between the lower bounds if $\alpha_1 = \sobolevint_1 - \frac{1}{2}$.
\begin{equation}\label{eq: lower bound dependence rkhs vs Holder}
    T\tilde{R}^{-\frac{\sobolevint_1-1/2}{\sobolevint_1+1/2}} \propto T\tilde{R}^{-\frac{\alpha_1}{\alpha_1+1}}.
\end{equation}
Therefore, for continuum-armed bandit problems, the statistical difficulty of adapting to kernel regularity of RKHS is the same as adapting to  H\"older exponents, if the H\"older exponents represent the smallest H\"older spaces that the RKHSs embed in.

\section{Discussion}\label{sec: discussion}
We discuss several future directions, stemming from the current limitations of our work. Our current theoretical results are for the domain with $\paramdim=1$,\footnote{Note that this is not equivalent to the dimension of the feature map of a kernel.} so it is of interest to extend the current results to $\paramdim>1$. \highlight{Instead of partitioning the domain $\domain=[0,1]$ into M sub-intervals, one needs to partition the hypercube $[0,1]^d$ into M sub-cubes and construct the hypothesis functions with appropriate Fourier decay correspondingly. Such an extension is possible akin to~\citet{scarlett2017lower}}.

Another direction is to derive adaptivity upper bounds in terms of Fourier decay as well and verify the tightness of the lower bound in more cases than Mat\'ern kernels. Since we currently investigate translation-invariant kernels, a more long-term direction is the investigation of adaptivity to rotation-invariant kernels, to connect to NTKs which are usually rotation-invariant dot-product kernels~\citep{bietti2020deep, chen2020deep, vakili2021uniform}. Finally, this study is of theoretical nature, so it remains an open problem to empirically study adaptivity to kernel regularity, based on the insights provided by our lower and upper bounds.  

\section*{Acknowledgements}
\highlight{This work is supported by NSF CCF1763734 and Simons Foundation grant.}
\bibliography{References}

\appendix
\onecolumn

\section{AUXILIARY}
\subsection{Proof of Norm Equivalency Between RKHS Norm and Sobolev Seminorm}\label{sec: proof of norm equivalence between RKHS and Sobolev seminorm}
\begin{proof}[Proof of Lemma~\ref{lemma: norm equivalency between semi Sobolev norm and RKHS norm}]
It is shown by~\citet[Theorem 10.12, Corollary 10.48]{wendland2004scattered} that: if a translation-invariant kernel $\kernel$ with Fourier decay rate $s$ (Lemma~\ref{lemma: norm equivalency between rkhs and Sobolev}), \highlight{then the associated RKHS $\rkhs_\kernel$ defined on a Lipschitz domain $\Omega$ is norm equivalent} to the Sobolev space $\sobolevspace^{\sobolevint=s, 2}(\Omega)$. The norm equivalency indicates that there exist two constants $c_1, c_2$, $0<c_1<c_2$, such that for $\func \in \rkhs_\kernel(\Omega)$, the following statement holds. 
\begin{equation}\label{eq: norm equivalence between RKHS and Sobolev norm}
    c_1\Vert \func \Vert_{m, 2, 
    \domain} \leq \Vert \func \Vert_{\rkhs_\kernel} \leq c_2 \Vert \func \Vert_{m, 2, \domain}.
\end{equation}
Now, we examine conditions under which the norm equivalency can be extended between the \emph{seminorm} (\eqref{eq: def seminorm Sobolev}) of Sobolev spaces and the RKHS norm. 
As in Lemma~\ref{lemma: norm equivalency between semi Sobolev norm and RKHS norm}, let $W_0^{m,p}(\domain)$ denote the closure of $C^\infty_0(\domain)$ in $W^{m,p}(\domain)$.\footnote{Here, we borrow the definitions from~\citet{adams2003sobolev}.} 
\citet[6.26]{adams2003sobolev} give the following result: if $\domain$ has finite width, then for $f\in W_0^{m,p}$, the seminorm $\vert{\cdot}\vert_{m,p}$ is equivalent to the standard norm $\Vert \cdot\Vert_{m,p}$. 
\highlight{The one-dimensional interval domain $\domain$ we consider trivially satisfies the Lipschitz boundary condition, hence we have the following result. }
\begin{lemma}\label{lemma: equal norm in Sobolev norms}
If a function lies in $W^{m,p}_0(\domain)$ where $\domain=[0,1]$, then there exists a constant \highlight{$K<\infty$}, such that
\begin{equation}
    \vert{\cdot}\vert_{m,p,\domain}\leq \Vert\cdot \Vert_{m,p,\domain} \leq K \vert{\cdot}\vert_{m,p,\domain}.
\end{equation}
\end{lemma}
Combining Lemma~\ref{lemma: equal norm in Sobolev norms} with the norm equivalency in~\eqref{eq: norm equivalence between RKHS and Sobolev norm}, \highlight{we recover the inequalities in Lemma~\ref{lemma: norm equivalency between semi Sobolev norm and RKHS norm}.}

\begin{equation}
    c_1\vert \func \vert_{m, 2, 
    \domain} \leq \Vert \func \Vert_{\rkhs_\kernel} \leq K c_2 \vert \func \vert_{m, 2, \domain}.
\end{equation}

\end{proof}

\subsection{The Full Statement of Theorem~\ref{thm: lower bound rkhs}}\label{sec: full lower bound rkhs}
We present the full version of Theorem~\ref{thm: lower bound rkhs}, which fully states the constraints on the radius $B_1$ and $B_2$ in Theorem~\ref{thm: lower bound rkhs}. The proof is deferred to Appendix~\ref{sec: proof of lower bound rkhs}.
\begin{theorem}\label{thm: lower bound rkhs, full version}
    Consider the bandit problem setting (Section~\ref{sec: problem setting}) with noises $\{\noise_t\}_{t=1\dots T}$ that are $\frac{1}{4}$-subgaussian. Further assume that the $\lspace_2$ norm of functions $f$ we consider is upper bounded by finite value $\gamma_0<\infty$: $\Vert \func\Vert_{2}\leq \gamma_0$. Let $\regupperbound$ be a positive number, let $\sobolevint_2>\sobolevint_1>0$ be two positive integers, and let $B_1, B_2$ be two positive variables that satisfy the following conditions.
    \begin{align}
        &\normequivu\max\left\{\frac{3^{\sobolevint_1 + \frac{1}{2}}}{32} C(\sobolevint_1)^{-\sobolevint_1+\frac{1}{2}} \regupperbound^{-1}, K(\sobolevint_1, \sobolevint_2,\gamma_0,\domain){\normequivu}^{\frac{\sobolevint_2-\sobolevint_1}{\sobolevint_2}} B_2^\frac{\sobolevint_1}{\sobolevint_2}\right\}\label{eq: constraint on radius, lower bound rkhs}\nonumber \\
        &\quad \quad \leq B_1\leq C'(\sobolevint_1,\sobolevint_2)^{-(\sobolevint_1+\frac{1}{2})}{\normequivu}^{(-\sobolevint_1+\frac{1}{2})} B_2^{\sobolevint_1+\frac{1}{2}} {\regupperbound}^{\sobolevint_1-\frac{1}{2}}
    \end{align}
    where $C(\sobolevint_1)$ and $C'(\sobolevint_1, \sobolevint_2)$ are constants whose exact forms are defined in~\eqref{eq: def of C(m1)} and~\eqref{eq: def of C(m1, m2)} in the proof. $K(\sobolevint_1, \sobolevint_2,\gamma_0,\domain)$ is a constant depending on $\sobolevint_1, \sobolevint_2$, the domain and $\gamma_0$.\footnote{The exact value of $K(\sobolevint_2,\gamma_0,\domain)$ is deferred to the proof of Theorem 4.14 in~\citet{adams2003sobolev}. } 

    Consider any algorithm that achieves in RKHS ball $\rkhs_{\kernelfourier{\sobolevint_2}}(\domain, B_2)$ the following regret upper bound, where the kernel $\kernelfourier{\sobolevint_2}$ has Fourier decay rate $\sobolevint_2$. 
    \begin{equation}
    \sup_{\func\in\rkhs_{\kernelfourier{\sobolevint_2}}(\domain, B_2)}\E[R_T] \leq \regupperbound, 
    \end{equation}
    then, the regret of this algorithm in a (less smooth) RKHS ball induced by another kernel $\kernelfourier{\sobolevint_1}$ with Fourier decay rate $\sobolevint_1$ is lower bounded by the following.
    \begin{align}
        &\sup_{\func\in \rkhs_{\kernelfourier{\sobolevint_1}}(\domain, B_1)}\E[R_T] \geq \frac{1}{8} \left(\frac{C(\sobolevint_1)}{32}\right)^{\frac{\sobolevint_1-1/2}{\sobolevint_1+1/2}}~ {\left(\frac{B_1}{\normequivu}\right)}^{\frac{1}{\sobolevint_1+1/2}}~\regupperbound^{-\frac{\sobolevint_1-1/2}{\sobolevint_1+1/2}}~T.
    \end{align}
\end{theorem}

\subsection{Adaptivity Upper Bound of RBBE}\label{sec: appendix rbbe details}
At each round, RBBE~\citep{pacchiano2020regret} first performs an elimination step to remove misspecified base algorithms, then selects a base algorithm among the remaining ones. The elimination step tests whether each base algorithm is well-specified, that is, whether each base algorithm's hypothesis space contains the underlying function. If a base algorithm fails the test, then it is eliminated. In the selection step, the master algorithm simply chooses the base algorithm with the smallest presumed cumulative pseudo-regret. Therefore, RBBE can be thought of as using a stochastic master algorithm (remarked in~\citet{pacchiano2020regret} as well), instead of using an adversarial one as CORRAL~\citep{agarwal2017corralling, pacchiano2020model} does. 

The general regret of RBBE is stated in terms of the play ratio, which is the ratio between the number of times a base algorithm is played and the number of times that the best base algorithm is played. To instantiate the play ratio, \citet{pacchiano2020model} considers only the setting where the regret rates of all base algorithms (if well-specified) have the same exponents on $T$. That is, the regret rates are $T^\beta$ with a \emph{fixed} $\beta\in (0,1]$ across all base algorithms. However, this setting does not align with our setting where, for base algorithm $i$ with input value $\nu_i$, the exponent of $T$ in its (well-specified) regret bound is $\frac{\nu_i+1}{2\nu_i+1}$. Hence, we make changes to the proof in~\citet{pacchiano2020model} to apply it to our problem setting. The result of RBBE is stated in Theorem~\ref{thm: RBBE with kernelised bandit} and the proof is deferred to Appendix~\ref{sec: proof of thm RBBE with kernelised bandit}.
\begin{theorem}\label{thm: RBBE with kernelised bandit}
    Suppose that the problem setting, the set of candidate values $\vu$ and the set of base algorithms $\sA$ are the same as defined in Theorem~\ref{thm: CORRAL with kernelised bandits}. 
    The regret of RBBE applied with base algorithms in $\sA$ is as follows, with high probability $1-\delta$. 
    \begin{equation}\label{eq: rbbe upper bound}
        \sup_{\func\in\rkhs_{\kernel_{\text{Mat\'ern},\nu^*}}} R_T = \tilde{O}(T^{\frac{1+4\nu^*+2{\nu^*}^2}{1+4\nu^* + 4{\nu^*}^2}}).
    \end{equation}
\end{theorem}

\section{PROOFS OF RESULTS}
\subsection{Proof of Theorem~\ref{thm: lower bound rkhs, full version}}\label{sec: proof of lower bound rkhs}
As explained in Section~\ref{sec: proof sketch lower bound}, the proof of Theorem~\ref{thm: lower bound rkhs, full version} arises from the proof of a parallel Sobolev version of the adaptivity lower bound. We formally state the Sobolev version of adaptivity lower bound below. 
\begin{theorem}\label{thm: lower bound Sobolev}
Consider the bandit problem setting (Section~\ref{sec: problem setting}) with noises $\{\noise_t\}_{t=1\dots T}$ that are $\frac{1}{4}$-subgaussian. Further assume that the $\lspace_2$ norms of functions $f$ we consider are upper bounded by the finite value $\gamma_0<\infty$: $\Vert \func\Vert_{2}\leq \gamma_0$. \footnote{By our assumption on the underlying function $\func$ in~\eqref{eq: definition Sobolev space}, we know that it has bounded $\lspace_2$ norm. } Let $\regupperbound$ be a positive number, let $\sobolevint_2>\sobolevint_1>0$ be two positive integers, and let $L_1, L_2$ be two positive variables that satisfy the following conditions:
\begin{align}
    &\max\left\{\frac{3^{\sobolevint_1 + \frac{1}{2}}}{32} C(\sobolevint_1)^{-\sobolevint_1+\frac{1}{2}} \regupperbound^{-1}, K(\sobolevint_1, \sobolevint_2,\gamma_0,\domain){\sobolevradius_2}^\frac{\sobolevint_1}{\sobolevint_2}\right\} \\
    &\qquad \leq L_1 \leq C'(\sobolevint_1,\sobolevint_2)^{-(\sobolevint_1+\frac{1}{2})} L_2^{\sobolevint_1+\frac{1}{2}} \regupperbound^{\sobolevint_1-\frac{1}{2}}\nonumber
\end{align}
where $C(\sobolevint_1)$, $C'(\sobolevint_1, \sobolevint_2)$ are constants whose exact forms are defined in~\eqref{eq: def of C(m1)} and~\eqref{eq: def of C(m1, m2)} respectively. $K(\sobolevint_1, \sobolevint_2,\gamma_0,\domain)$ is a constant depending on $\sobolevint_1, \sobolevint_2$, the domain and $\gamma_0$, the upper bound on the $\lspace_2$ norm of functions in the Sobolev ball.\footnote{The exact value of $K(\sobolevint_2,\gamma_0,\domain)$ is deferred to the proof of Theorem 4.14 in~\citet{adams2003sobolev}. } 
Consider an algorithm that achieves in the Sobolev ball $\sobolevspace^{\sobolevint_2}(\domain, L_2)$ a regret upper bound of $\regupperbound$. 
\begin{equation}
     \sup_{\func\in \sobolevspace^{\sobolevint_2, 2}(\domain, L_2)} \E[R_T] \leq \regupperbound,
\end{equation}
then, the regret of this algorithm in the less-smooth Sobolev ball $\sobolevspace^{\sobolevint_1}(\domain, L_1)$ is lower bounded by the following.
\begin{align}
    &\sup_{\func\in \sobolevspace^{\sobolevint_1}(\domain, L_1)} \E[R_T] \geq \frac{1}{8} \left(\frac{C(\sobolevint_1)}{32}\right)^{\frac{\sobolevint_1-1/2}{\sobolevint_1+1/2}}~{L_1}^{\frac{1}{\sobolevint_1+1/2}}~\regupperbound^{-\frac{\sobolevint_1-1/2}{\sobolevint_1+1/2}}~T.
\end{align}
\end{theorem}

In the next part, we present the proof of Theorem~\ref{thm: lower bound Sobolev}, which also leads to Theorem~\ref{thm: lower bound rkhs, full version}. The values $B_1, B_2$ in Theorem~\ref{thm: lower bound rkhs} should be set as follows.
\begin{align}\label{eq: relation between B12 and L12}
    B_1 = \normequivu L_1, B_2 = \normequivu L_2,
\end{align}
where $\normequivu$ is the global constant in Lemma~\ref{lemma: norm equivalency between semi Sobolev norm and RKHS norm}.

\begin{proof}
Consider the Sobolev version of the theorem (Theorem~\ref{thm: lower bound Sobolev}). Recall that the adaptivity is between balls in two different spaces, the ``rougher'' space $\sobolevspace^{\sobolevint_1}(\domain, L_1)$ and the ``smoother'' space $\sobolevspace^{\sobolevint_2}(\domain, L_2)$. First, we consider the constraints between $L_1$ and $L_2$ such that $\sobolevspace^{\sobolevint_2}(\domain, L_2) \subset \sobolevspace^{\sobolevint_1}(\domain, L_1)$. In other words, $\func \in \sobolevspace^{\sobolevint_2}(\domain, L_2)$ should be sufficient condition for $\func\in \sobolevspace^{\sobolevint_1}(\domain, L_1)$. Theorem 4.14 in~\citet{adams2003sobolev} and references therein give the following interpolation upper bound between orders of smoothness for a function $\func\in \sobolevspace^{\sobolevint_1}(\domain)$,
\begin{equation}\label{eq: interpolation between Sobolev seminorms}
    \vert \func \vert_{\sobolevint_1,2}\leq K(\sobolevint_2,\domain) {(\vert \func \vert_{\sobolevint_2})}^\frac{\sobolevint_1}{\sobolevint_2} ~ \Vert \func \Vert_{2}^\frac{\sobolevint_2 - \sobolevint_1}{\sobolevint_2},
\end{equation}
where $K(\sobolevint_2,\domain)$ is a constant depending only on $\sobolevint_2$ and the domain $\domain$.
If $\func \in \sobolevspace^{\sobolevint_2}(\domain, L_2)$, then by definition (\eqref{eq: definition Sobolev ball}) we know that $\vert \func \vert_{\sobolevint_2}\leq L_2$. Using~\eqref{eq: interpolation between Sobolev seminorms}, we have that:
\begin{align}\label{eq: interpolation between Sobolev seminorms pt 2}
    \vert \func \vert_{\sobolevint_1,2} \leq K(\sobolevint_2,\domain) {\sobolevradius_2}^\frac{\sobolevint_1}{\sobolevint_2} ~ \Vert \func \Vert_{2}^\frac{\sobolevint_2 - \sobolevint_1}{\sobolevint_2}
\end{align}
To ensure that the two Sobolev balls are nested, $L_1$ should be larger than the right-hand side of the above inequality. The $\lspace_2$ norm of $\func$ is upper bounded by $\Vert \func\Vert_2\leq \gamma_0$. Plugging it in~\eqref{eq: interpolation between Sobolev seminorms pt 2} incurs an lower bound for $\sobolevradius_1$: 
\begin{equation*}
    \sobolevradius_1\geq K(\sobolevint_1, \sobolevint_2,\gamma_0,\domain){\sobolevradius_2}^\frac{\sobolevint_1}{\sobolevint_2}=K(\sobolevint_2,\domain) {\gamma_0}^\frac{\sobolevint_2 - \sobolevint_1}{\sobolevint_2}{\sobolevradius_2}^\frac{\sobolevint_1}{\sobolevint_2}.
\end{equation*}

Having established $\sobolevspace^{\sobolevint_2}(\domain, L_2) \subset \sobolevspace^{\sobolevint_1}(\domain, L_1)$, we start with the formal proof of the adaptivity lower bound. 

\textbf{Function Construction Part I.}
This part is adapted from the regression lower bounds in~\citet[Section 2.6]{tsybakov2004introduction}.
Let $\numhypo$ be a positive integer parameter, which is the number of hypothesis functions we need. The value of $\numhypo$ remains to be determined later in the proof. In the following, we shall assume $\numhypo\geq 2$ and eventually prove that this assumption holds. Further, define bandwidth $h =\frac{1}{2\numhypo}$. Let $\Delta>0$ be a parameter that represents the maximum of the $M$ hypothesis functions in $\sobolevspace^{\sobolevint_1, 2}(\domain, L_1)$. The value of $\Delta$ remains to be 
determined later in the proof same as $M$. 

Partition the $1$-dimensional domain $\domain=[0, 1]$ into $\numhypo+1$ bins: $H_{0\dots \numhypo}$, such that $\cup_{s=0\dots \numhypo}H_s = \domain$. Define the bins and their middle points $\bar{\action}_{0, \dots \numhypo}$ as follows.
\begin{align*}
 &H_s = \left[\frac{s-1}{2\numhypo}, \frac{s}{2\numhypo}\right], ~\bar{x}_s=\frac{s-\frac{1}{2}}{2\numhypo}, ~\text{for}~s = 1\dots \numhypo, \\
 &H_0 = \left[\frac{1}{2}, 1\right], ~\bar{\action}_0 = \frac{3}{4}.
\end{align*}

We use the bump function as a base function, then we shift the base function to construct the hypothesis functions. The bump function is defined as follows. It has compact support on $(-1,1)$. Function $K_0(\cdot)$ is infinitely differentiable with continuous derivatives~\citep[(2.34)]{tsybakov2004introduction}.
\begin{equation}
    K_0(\action) = \exp(\frac{-1}{1-\action^2}) \sI(\vert\action\vert<1). 
\end{equation}

Next, define $\numhypo+1$ functions as follows, each one has support inside one of the $\numhypo+1$ bins. 
\begin{align}
    &f_s = a h^{\sobolevint_1-\frac{1}{2}} K(\frac{\action-\bar{\action}_s}{h}), ~~s = 1\dots \numhypo, \\
    &f_0 = \tilde{a} h^{\sobolevint_2-\frac{1}{2}} \tilde{K}(\frac{\action-\bar{\action}_0}{h}), 
\end{align}
where
\begin{align}
    &K(u) = K_0(b u), \\
    &\tilde{K}(u) = K_0(\tilde{b} u).
\end{align}
$a, b, \tilde{a}, \tilde{b}$ are non-negative parameters to be defined later. We require that $b\geq 2$ and $\tilde{b}\geq 4h$, so that the support of every function $f_s$ is inside $H_s$, $\forall s\leq \numhypo$. Lemma~\ref{lemma: delta and M} ensures that the requirements on $b,\tilde{b}$ hold, by posing constraints between $\Delta$ and $\numhypo$. 

We introduce the following lemma to specify requirements on the variables $a, b, \tilde{a}, \tilde{b}$, with respect to $\Delta$ and $L_1, L_2$. This is to make sure that values of $a, b, \tilde{a}, \tilde{b}$ guarantee that $f_{s} \in \sobolevspace^{\sobolevint_1}(\domain, L_1), ~\forall 1\leq s\leq \numhypo$ and $f_0\in \sobolevspace^{\sobolevint_2}(\domain, L_2)$. 

\begin{lemma}\label{lemma: function params}
Let $K_0^*$ to denote the maximum value of $K_0(\cdot)$, a constant less than $1$. Let $I_{\sobolevint_1}, I_{\sobolevint_2}$ denote the $\lspace_2$ norms of the $\sobolevint_1, \sobolevint_2$-th order derivatives of $K_0(\cdot)$, respectively. That is, $I_{\sobolevint_1} = \int_{-1}^1 \left[K_0^{(\sobolevint_1)}(u)\right]^2 du$ and $I_{\sobolevint_2} = \int_{-1}^1 \left[K_0^{(\sobolevint_2)}(u)\right]^2 du$. 
Then, if $\Delta$ is the maximum of $f_s$ in $\sobolevspace^{\sobolevint_1, 2}(\domain, L_1)$, for all $s = 1\dots M$ and $\Delta/2$ is the maximum of $f_0$ in $\sobolevspace^{\sobolevint_2, 2}(\domain, L_2)$, the function parameters $a, b, \tilde{a}, \tilde{b}$ satisfy the following:
\begin{align}
    &a = \Delta (2\numhypo)^{\sobolevint_1-\frac{1}{2}} / K_0^* \\
    &\tilde{a} = \Delta (2\numhypo)^{\sobolevint_2-\frac{1}{2}} / 2K_0^* \\
    &b\leq \left( \frac{L_1^2 {K_0^*}^2}{\Delta^2 (2\numhypo)^{2\sobolevint_1-1} I_{\sobolevint_1}} \right)^{\frac{1}{2\sobolevint_1 - 1}} \\
    &\tilde{b}\leq \left( \frac{4 L_2^2 {K_0^*}^2}{\Delta^2 (2\numhypo)^{2\sobolevint_2-1} I_{\sobolevint_2}} \right)^{\frac{1}{2\sobolevint_2 - 1}},
\end{align}
\end{lemma}

\begin{proof}[Proof of Lemma~\ref{lemma: function params}]
The constraints on $a, \tilde{a}$ follows trivially from the requirement that $f_s^* = \Delta$ for $s = 1\dots \numhypo$, $f_0^* = \Delta/2$, and plugging in $h = 1/2\numhypo$. 

The constraints on $b, \tilde{b}$ are to ensure that 
\begin{align*}
    &\norm{f_s^{(\sobolevint_1)}}\leq L_1, ~~s = 1\dots \numhypo \\
    &\norm{f_0^{(\sobolevint_2)}}\leq L_2
\end{align*}

We first consider requirement for $\norm{f_s^{(\sobolevint_1)}}\leq L_1, ~~s = 1\dots \numhypo$. For $s\geq 1$, 
\begin{align*}
    &\norm{f_s^{(\sobolevint_1)}}^2 \\
    &= \int_{0}^1 \left[f^{(\sobolevint_1)}(x)\right]^2 dx \\
    & =\int_{0}^1 \left[ah^{\sobolevint_1-\frac{1}{2}} \frac{\partial^{\sobolevint_1}}{\partial x^{\sobolevint_1}} \left(K(\frac{x - \bar{x}_s}{h})\right)\right]^2 dx\\
    &=a^2h^{2\sobolevint_1 - 1}\int_0^1 \left[\frac{\partial^{\sobolevint_1}}{\partial x^{\sobolevint_1}}\left(K_0\left(\frac{b}{h}(x-\bar{x}_s)\right)\right) \right]^2 dx \\
    &=a^2 h^{2\sobolevint_1 - 1}\int_0^1 \left[ (\frac{b}{h})^{\sobolevint_1}K_0^{(\sobolevint_1)}\left(\frac{b}{h}(x-\bar{x}_s)\right) \right]^2 dx \\
    & \stackrel{u = \frac{b}{h}(x-\bar{x}_s)}{=} a^2 h^{-1} b^{2\sobolevint_1} \int_{\frac{b}{h}(-\bar{x}_s)}^{\frac{b}{h}(1-\bar{x}_s)} \left[ K_0^{(\sobolevint_1)}(u)\right]^2 \frac{h}{b} du \\
    & = a^2 b^{2\sobolevint_1 - 1} \int_{-1}^1 \left[ K_0^{(\sobolevint_1)}(u)\right]^2 du = a^2 b^{2\sobolevint_1 -1} I_{\sobolevint_1}.
\end{align*}
The second to last step follows because the bump function $K_0$ has compact support on $(-1, 1)$ and the upper and lower limits of the integral satisfy:
\begin{align*}
    & \frac{b}{h}(1-\bar{x}_s) = b(\frac{1}{h}-s+\frac{1}{2}) > 1, \\
    &\frac{b}{h}(-\bar{x}_s) = -b(s-\frac{1}{2})\leq -1.
\end{align*}

Therefore, for $\norm{f_s^{(\sobolevint_1)}}^2\leq L_1^2$ to hold, we need $a^2 b^{2\sobolevint_1 -1} I_{\sobolevint_1}\leq L_1^2$.  This leads to
\begin{align}
    b&\leq \left(\frac{L_1^2}{a^2 I_{\sobolevint_1}} \right)^{\frac{1}{2\sobolevint_1-1}} \\
    &= \left(\frac{L_1^2 (K_0^*)^2}{\Delta^2 (2M)^{2\sobolevint_1-1}I_{\sobolevint_1}} \right)^{\frac{1}{2\sobolevint_1-1}}.
\end{align}

Similarly, for $s=0$, we have the following. 
\begin{align*}
    &\norm{f_s^{(\sobolevint_2)}}^2 \\
    &= \int_{0}^1 \left[f^{(\sobolevint_2)}(x)\right]^2 dx \\
    & =\int_{0}^1 \left[\tilde a h^{\sobolevint_2-\frac{1}{2}} \frac{\partial^{\sobolevint_2}}{\partial x^{\sobolevint_2}} \left(\tilde{K}(\frac{x - \bar{x}_0}{h})\right)\right]^2 dx \\
    &= \int_{0}^1 \tilde{a}^2 h^{2\sobolevint_2-1} \left[ \frac{\partial^{\sobolevint_2}}{\partial x^{\sobolevint_2}} \left(K_0(\frac{\tilde{b}(x - \bar{x}_0}{h}))\right)\right]^2 dx \\
    &= \tilde{a}^2 h^{2\sobolevint_2-1} \int_0^1 \left[\left(\frac{\tilde{b}}{h}\right)^{\sobolevint_2} K_0^{(\sobolevint_2)}(\frac{\tilde{b}}{h}(x-\bar{x}_0))\right]^2 dx \\
    &\stackrel{u =\tilde{b} (x-\bar{x_0})/h}{=}\tilde{a}^2 \tilde{b}^{2\sobolevint_2-1}\int_{-\frac{3\tilde{b}}{4h}}^{\frac{\tilde{b}}{h}(1-\frac{3}{4})} \left[K_0^{(\sobolevint_2)}(u)\right]^2 du \\
    &=\tilde{a}^2 \tilde{b}^{2\sobolevint_2-1}\int_{-1}^1 \left[K_0^{(\sobolevint_2)}(u)\right]^2 du \\
    &=\tilde{a}^2\tilde{b}^{2\sobolevint_2 - 1} I_{\sobolevint_2}. 
\end{align*}
Note that in the third last equation, the integral upper and lower limit satisfy: 
\begin{equation*}
    \frac{\tilde{b}}{h}(1-\frac{3}{4})> 1, \quad -\frac{3\tilde{b}}{4h}< -1.
\end{equation*}
For the above $\norm{f_s^{(m_2)}}^2$ to be less or equal to $L_2^2$, we need:
\begin{align}
    \tilde{b}\leq \left( \frac{L_2^2}{\tilde{a}^2 I_{\sobolevint_2}}\right)^{\frac{1}{2\sobolevint_2 - 1}} = \left(\frac{4L_2^2 (K_0^*)^2}{\Delta^2 (2M)^{2\sobolevint_2-1} I_{\sobolevint_2}} \right)^{\frac{1}{2\sobolevint_2 - 1}}
\end{align}
\end{proof}

Combining Lemma~\ref{lemma: function params} with what we required of the function parameters: $b\geq 2$ and $\tilde{b}\geq 4h$, we then need the following requirements for the parameter $\Delta$. Intuitively, the following lemma says that the functions cannot be too ``wavy'', so that they stay within the corresponding balls in Sobolev spaces. 

\begin{lemma}\label{lemma: delta and M}
For $b\geq 2, \tilde{b}\geq 4h$ to hold, $\Delta$ needs to satisfy the following constraints with respect to $\numhypo$ and the smoothness constants $L_1, L_2$. 

\begin{align}
    &\Delta/L_1 \leq \frac{K_0^*}{2^{2\sobolevint_1 - 1} M^{\sobolevint_1-\frac{1}{2}}\sqrt{I_{\sobolevint_1}}},\\
    &\Delta/L_2 \leq \frac{K_0^*}{2^{2\sobolevint_2-2}\sqrt{I_{\sobolevint_2}}}.
\end{align}
\end{lemma}
\begin{proof}[Proof of Lemma~\ref{lemma: delta and M}]
First, consider function $f_s$ when $s\geq 1$. Using the conclusions in Lemma~\ref{lemma: function params} we need the following,
\begin{equation*}
    \frac{L_1^2 (K_0^*)^2}{\Delta^2 (2\numhypo)^{2\sobolevint_1-1}I_{\sobolevint_1}} \geq b^{2\sobolevint_1-1}\geq 2^{2\sobolevint_1-1}.
\end{equation*}
What directly follows is the constraint on $\Delta$:
\begin{equation}
    \Delta^2 \leq \frac{L_1^2 (K_0^*)^2}{2^{4\sobolevint_1 -2}\numhypo^{2\sobolevint_1 - 1} I_{\sobolevint_1}}.
\end{equation}

Similarly, for $f_0$, we need
\begin{align*}
    \frac{L_2^2 (K_0^*)^2}{\Delta^2 (2\numhypo)^{2\sobolevint_2-1}I_{\sobolevint_2}} \geq \tilde{b}^{2\sobolevint_2-1}&\geq (4h)^{2\sobolevint_2-1} 
    = 2^{2\sobolevint_2 - 1}\numhypo^{1-2\sobolevint_2}.
\end{align*}
This leads to second constraint on $\Delta$:
\begin{equation}
    \Delta^2 \leq \frac{L_2^2 (K_0^*)^2}{2^{4\sobolevint_2 - 4}I_{\sobolevint_2}}.
\end{equation}
\end{proof}

\textbf{Function Construction Part II.} We have defined $f_0 \dots f_\numhypo$ in Part I, and identified the constraints between the floating parameters $\numhypo$ and $\Delta$, with respect to given parameters $\sobolevint_1, \sobolevint_2, L_1, L_2$ and known constants $K_0^*, I_{\sobolevint_1}, I_{\sobolevint_2}$. In this second part, we define $\numhypo+1$ bandit problems by defining their reward functions $\phi_s, ~s = 0\dots \numhypo$ in the following way:
\begin{align}
    &\phi_0 = f_0, \\
    &\phi_s = f_s + f_0, ~\forall 1\leq s\leq \numhypo.
\end{align}
It is obvious that the reward functions satisfy the following conditions. The conditions below are the Sobolev version. They are necessary for the latter half of this proof. Similar conditions were required in~\citet{locatelli2018adaptivity, hadiji2019polynomial}, see below for details.
\begin{enumerate}
    \item \label{adaptivity proof: construction condition: peak} The function $\phi_0$ has peak value $\Delta/2$ and functions $\phi_s, 1\leq s\leq \numhypo$ all have peak value $\Delta$. 
    \item \label{adaptivity proof: construction condition: smoothness} The function $\phi_0 \in \sobolevspace^{\sobolevint_2, 2}(\domain, L_2)$ and functions $\phi_s \in \sobolevspace^{\sobolevint_1, 2}(\domain, L_1), 1\leq s\leq M$. 
    \item \label{daptivity proof: construction condition: shape} For $s\geq 1$, $\phi_s(x) = \phi_0(x)$ for $x\notin H_s$. Also, $\phi_s^* - \phi_s(x)\geq \frac{\Delta}{2}$ when $x\notin H_s$. Here $\phi_s^* = \max_{x\in \domain}\phi_s(x)$. 
\end{enumerate}

\textbf{RKHS Version of the Proof.} 
We have now defined $\numhypo+1$ hypothesis functions in two balls in two different Sobolev spaces. By (i)the norm equivalency between Sobolev seminorm (Lemma~\ref{lemma: norm equivalency between semi Sobolev norm and RKHS norm}) and the RKHS norm; and (ii) the relationships between $B_1, L_1$ and $B_2, L_2$ in~\eqref{eq: relation between B12 and L12}, the reward functions also satisfy the following conditions. The conditions below are the RKHS version.  
\begin{enumerate}
    \item The function $\phi_0$ has peak value $\Delta/2$ and functions $\phi_s, 1\leq s\leq \numhypo$ all have peak value $\Delta$. 
    \item $\phi_0 \in \rkhs_{\kernelfourier{\sobolevint_2}}(\domain, B_2)$, $\phi_s \in \rkhs_{\kernelfourier{\sobolevint_1}}(\domain, B_1)$, for $1\leq s\leq M$. 
    \item $\forall s\geq 1$, $\phi_s(x) = \phi_0(x)$ when $x\notin H_s$. Also, $\phi_s^* - \phi_s(x)\geq \frac{\Delta}{2}$ when $x\notin H_s$. 
\end{enumerate}

\textbf{Lower Bounding Cumulative Regret (Proof Sketch).}
This part shows the cumulative regret of an algorithm on functions $\phi_1\dots \phi_\numhypo$ is lower bounded by a rate that depends reversely on $\regupperbound$, if this algorithm has a regret upper bound of $\tilde{R}$ on reward function $\phi_0$. The proof in the following directly follows from~\citet{hadiji2019polynomial} and relies on Pinsker's inequality. 
We write down a proof sketch here for completeness, readers interested in the full version can refer to~\citet[Section $F$]{hadiji2019polynomial}. 
We use their notations in this part unless otherwise specified. Those include $N_{H_s}(T)$ which is the number of times an algorithm selects an action in bin $H_s$; $\sP_s^T(\cdot)$ which is the probability distribution of trajectory $\{x_t, y_t\}_{t=1\dots T}$, when the reward function in the bandit setting is defined by $\phi_s$, for $0\leq s\leq M$. Similarly, $\E_s[\cdot]$ is the expectation with respect to probability $\sP_s$. 

By definitions of the reward functions, when the underlying function is $\phi_s$ for some $s\geq 1$, the cumulative regret is lower bounded by
\begin{equation}
    R_{T,s} \geq \frac{\Delta}{2}(T - \E_s[N_{H_s}(T)])
\end{equation}

For $s=0$, the regret is lower bounded by
\begin{equation}
    R_{T,0}\geq \frac{\Delta}{2}\sum_{s'=1}^\numhypo \E_0[N_{H_{s'}}(T)].
\end{equation}

Pinsker's inequality is used to establish a relationship between the two lower bounds defined above. The \eqref{eq: adaptivity lower bound proof: use pinskers} is a core step of the proof. 
\begin{equation}\label{eq: adaptivity lower bound proof: use pinskers}
    \frac{1}{T}\E_s[N_{H_s}(T)] - \frac{1}{T}\E_0[N_{H_s}(T)] \leq \sqrt{\frac{1}{2}\KL(\sP_0^T, \sP_s^T)}.
\end{equation}
Calculation of KL distance $\KL(\cdot, \cdot)$ relies on condition~\ref{daptivity proof: construction condition: shape} of $\phi_{0\dots \numhypo}$, as well as the assumption that the noise is $1/4$-subgaussian. The result is that the KL distance is bounded by the following.
\begin{equation}
    \KL(\sP_0^T, \sP_s^T) = 2\E_0 [N_{H_s}(T)]\Delta^2.
\end{equation}

With the above, a key intermediate result is reiterated below.
\begin{equation}
    \frac{1}{\numhypo}\sum_{s=1}^\numhypo R_{T,s} \geq \frac{T}{2} \Delta\left(1-\frac{1}{\numhypo}-\sqrt{\frac{\Delta\cdot R_{T,0}}{\numhypo}}\right).
\end{equation}

Recall that our Theorem~\ref{thm: lower bound Sobolev} assumes that $\sup_{\func\in \sobolevspace^{\sobolevint_2, 2}(\domain, L_2)} \mathit{R}_T \leq \regupperbound$, and since $\phi_0\in \sobolevspace^{\sobolevint_2, 2}(\domain, L_2)$, it follows directly that $R_{T,0}\leq \regupperbound$. Therefore, the above inequality becomes
\begin{align*}
    \frac{1}{\numhypo}\sum_{s=1}^\numhypo R_{T,s} &\geq \frac{T}{2} \Delta\left(1-\frac{1}{\numhypo}-\sqrt{\frac{\Delta\cdot R_{T,0}}{\numhypo}}\right) \\
    &\geq \frac{T}{2} \Delta\left(\frac{1}{2}-\sqrt{\frac{\Delta \regupperbound}{\numhypo}}\right).
\end{align*}
In the last inequality, $\numhypo\geq2$ is used. This assumption is \emph{not} violated, as shown later. 

\textbf{Choosing the Appropriate value for $\Delta$.}
Following the above lower bound, we need to choose a value for $\Delta$ that (i) does not violate any of the requirements (Lemma~\ref{lemma: delta and M}) and (ii) maximizes/tightens the lower bound. To do so, the value of $\Delta$ should satisfy: 
\begin{enumerate}
    \item $\sqrt{\frac{\Delta \regupperbound}{\numhypo}}\leq \frac{1}{4}$, where $\frac{1}{4}$ is a constant less than $\frac{1}{2}$ (chosen in an arbitrary manner). 
    \item $\Delta / L_1 \leq \frac{(K_0^*)}{2^{2\sobolevint_1-1}\numhypo^{\sobolevint_1-\frac{1}{2}} I_{\sobolevint_1}^{\frac{1}{2}}}$. Note that this condition satisfies only half of the requirements in Lemma~\ref{lemma: delta and M}. We later show that the other condition in Lemma~\ref{lemma: delta and M} is also satisfied with the selected $\Delta$. 
\end{enumerate}

When maximizing $\Delta$, we first set $\Delta / L_1 \approx \frac{(K_0^*)}{2^{2\sobolevint_1-1}\numhypo^{\sobolevint_1-\frac{1}{2}} I_{\sobolevint_1}^{\frac{1}{2}}}$ to achieve the optimal trade-off between $\numhypo$ and $\Delta$. That is, we set 
\begin{equation}
    \numhypo = \left\lfloor \left(\frac{L_1 K_0^*}{2^{2\sobolevint_1-1} I_{\sobolevint_1}^{\frac{1}{2}} \Delta} \right)^\frac{1}{\sobolevint_1-\frac{1}{2}}\right\rfloor,
\end{equation}
since $\numhypo$ needs to be an integer. By simplifying the constant term:
\begin{equation}\label{eq: def of C(m1)}
    C(\sobolevint_1) \stackrel{\triangle}{=} (\frac{K_0^*}{2^{2\sobolevint_1 - 1}I_{\sobolevint_1}^\frac{1}{2}}),
\end{equation}
we get a simpler expression of $M$:
\begin{equation}\label{eq: adaptivity proof: M}
    \numhypo = \left\lfloor C(\sobolevint_1)~L_1^\frac{2}{2\sobolevint_1 - 1} \Delta^\frac{-2}{2\sobolevint_1-1}\right\rfloor.
\end{equation}

If $\Delta \regupperbound /\left( C(\sobolevint_1) L_1^\frac{2}{2\sobolevint_1 - 1} \Delta^\frac{-2}{2\sobolevint_1-1}\right) \leq \frac{1}{32}$, the condition $\sqrt{\frac{\Delta \regupperbound}{\sobolevint}}\leq \frac{1}{4}$ would be satisfied, using the fact that $\frac{x}{2}\leq \floor{x}, \forall x>2$. Shuffling some terms, the requirement $\Delta \regupperbound/\left( C(\sobolevint_1) L_1^\frac{2}{2\sobolevint_1 - 1} \Delta^\frac{-2}{2\sobolevint_1-1}\right) \leq \frac{1}{32}$ becomes:
\begin{align*}
    &\Delta\leq \frac{1}{32}C(\sobolevint_1)L^\frac{2}{2\sobolevint_1 - 1}\Delta^\frac{-2}{2\sobolevint_1 - 1} B^{-1} \\
    &\Delta^\frac{2\sobolevint_1 + 1}{2\sobolevint_1-1} \leq \frac{C(\sobolevint_1)}{32} L_1^\frac{2}{2\sobolevint_1 -1}\regupperbound^{-1} \\
    &\Delta \leq \left(\frac{C(\sobolevint_1)}{32} \right)^\frac{\sobolevint_1 - \frac{1}{2}}{\sobolevint_1 + \frac{1}{2}} L_1^{\frac{1}{\sobolevint_1+\frac{1}{2}}} \regupperbound^{-\frac{\sobolevint_1-\frac{1}{2}}{\sobolevint_1 + \frac{1}{2}}}.
\end{align*}
To maximize $\Delta$, we thereby choose
\begin{equation}\label{eq: adaptivity proof: Delta}
    \Delta = \left(\frac{C(\sobolevint_1)}{32} \right)^\frac{\sobolevint_1 - \frac{1}{2}}{\sobolevint_1 + \frac{1}{2}} L_1^{\frac{1}{\sobolevint_1+\frac{1}{2}}} \numhypo^{-\frac{\sobolevint_1-\frac{1}{2}}{\sobolevint_1 + \frac{1}{2}}}.
\end{equation}
This leads to the final lower bound:
\begin{align}
    &\frac{1}{\numhypo}\sum_{s=1}^\numhypo R_{T,s} \nonumber\\
    &\geq \frac{T}{2} \Delta\left(\frac{1}{2}-\sqrt{\frac{\Delta \regupperbound}{\numhypo}}\right)\nonumber \geq \frac{T\Delta}{8}\nonumber\\
    &= \frac{1}{8} \left(\frac{C(\sobolevint_1)}{32}\right)^\frac{\sobolevint_1 -1/2}{\sobolevint_1 + 1/2} T L^\frac{1}{\sobolevint_1 + 1/2} \regupperbound^{-\frac{\sobolevint_1 - 1/2}{\sobolevint_1 + 1/2}}.
\end{align}

\textbf{Verify Assumptions.} Last but not least, we have to make sure that the assumptions made throughout the proof are satisfied, by our choice of $\Delta$ in~\eqref{eq: adaptivity proof: Delta} and $\numhypo$ in~\eqref{eq: adaptivity proof: M}. 
\begin{enumerate}
    \item $\numhypo\geq 2$. By the definition of $\numhypo$ in~\eqref{eq: adaptivity proof: M}, we need to ensure that $ C(\sobolevint_1)~L_1^\frac{2}{2\sobolevint_1 - 1} \Delta^\frac{-2}{2\sobolevint_1-1} \geq 2+1=3$. Further, plugging in~\eqref{eq: adaptivity proof: Delta}, this becomes the following requirement of $L_1$:
    \begin{equation}\label{eq: adaptivity proof: L1 lower}
        L_1\geq \frac{3^{\sobolevint_1 + \frac{1}{2}}}{32} C(\sobolevint_1)^{-\sobolevint_1+\frac{1}{2}} \regupperbound^{-1}. 
    \end{equation}
    \item $\Delta/L_2 \leq \frac{K_0^*}{2^{2\sobolevint_2-2}\sqrt{I_{\sobolevint_2}}} $. This is the second requirement in Lemma~\ref{lemma: delta and M} that has not yet been verified to hold. For this condition to hold, the following constraint on $L_2$ should be met.
    \begin{equation}
        L_2 \geq C'(\sobolevint_1, \sobolevint_2) L_1^\frac{1}{\sobolevint_1 + 1/2} \regupperbound^{-\frac{\sobolevint_1 - 1/2}{\sobolevint_1 + 1/2}},
    \end{equation}
    where,
    \begin{equation}\label{eq: def of C(m1, m2)}
    C'(\sobolevint_1, \sobolevint_2) = 2^{2\sobolevint_2-2}\left(\frac{C(\sobolevint_1)}{32}\right)^\frac{\sobolevint_1-1/2}{\sobolevint_1 + 1/2} \frac{\sqrt{I_{\sobolevint_2}}}{K_0^*}
    \end{equation}
    is a constant (independent of $T$) that depends on $\sobolevint_1, \sobolevint_2$. In other words, to make sure that the requirements in Lemma~\ref{lemma: delta and M} are met, we need in the assumptions the following constraint. 
    \begin{equation}
        L_1 \leq C'(\sobolevint_1,\sobolevint_2)^{-(\sobolevint_1+\frac{1}{2})} L_2^{\sobolevint_1+\frac{1}{2}} \regupperbound^{\sobolevint_1-\frac{1}{2}}.
    \end{equation}
\end{enumerate}
We have proved Theorem~\ref{thm: lower bound Sobolev} (Sobolev version).

The constraints on $B_1$ and $B_2$ in Theorem~\ref{thm: lower bound rkhs} are derived from the constraints on $L_1, L_2$ in Theorem~\ref{thm: lower bound Sobolev} and setting $B_1, B_2$ as instructed in~\eqref{eq: relation between B12 and L12}. 
Then the proof of Theorem~\ref{thm: lower bound rkhs} is also completed.
\end{proof}
\subsection{Proof of Corollary~\ref{thm: lower bound with Matern kernel}}\label{sec: proof of corollary of lower bound with Matern kernels}
When $\paramdim=1$, Mat\'ern kernel with regularity parameter $\nu$ has Fourier decay rate of $\nu+\frac{1}{2}$ (Definition~\ref{def: matern kernels}). The algorithm considered in Corollary~\ref{thm: lower bound with Matern kernel} thus satisfies the regret upper bound on an RKHS induced by a kernel with decay rate $\sobolevint_2 = \nu_2 + \frac{1}{2}$ which is $\tilde{R} = \tilde{O}(T^{\frac{\sobolevint_2+\frac{1}{2}}{2\sobolevint_2}})$. Let $\sobolevint_1$ be an integer larger than $\sobolevint_2$. Applying Theorem~\ref{thm: lower bound rkhs}, the lower bound on RKHS of a kernel with Fourier decay rate $\sobolevint_1$ is $\Omega(\tilde{R}^{-\frac{\sobolevint_1-\frac{1}{2}}{\sobolevint_1+\frac{1}{2}}} T)$. For simplicity, we omit the dependence on $B$ (and constant factors) and focus only on the dependence on $T$.
Plugging in the rate of $\tilde{R}$, the lower bound then becomes $\Omega(T^\frac{\sobolevint_1\sobolevint_2+\frac{3}{2}\sobolevint_2-\frac{1}{2}\sobolevint_1+\frac{1}{4}}{2\sobolevint_1\sobolevint_2+\sobolevint_2})$. Set $\sobolevint_1 = \nu_1 + \frac{1}{2}$ as the Fourier decay rate of $\kernel_{\text{Mat\'ern},\nu_1}$ in Corollary~\ref{thm: lower bound with Matern kernel}. Then, we get the lower bound by substituting $\sobolevint_2 = \nu_2+\frac{1}{2}$ and $\sobolevint_1 = \nu_1+\frac{1}{2}$, which is $\Omega(T^{\frac{\nu_1\nu_2+2\nu_2+1}{(\nu_1+1)(2\nu_2+1)}})$.

\subsection{Proof of Theorem~\ref{thm: UCB-Meta regret in RKHS}}\label{sec: proof of UCB-Meta regret in RKHS}
UCB-Meta~\citep{liu2021smooth} achieves minimax regret rate in dependence on $T$ (except log factors) in H\"older spaces with H\"older exponent $\alpha>1$. For $0<\alpha\leq 1$, it reduces to the minimax optimal continuum-armed bandit algorithm from~\citet{auer2007improved}. For simplicity, we consider UCB-Meta as the general algorithm for continuum-armed bandits in H\"older spaces. To prove that it is also minimax optimal over RKHS of certain Mat\'ern kernels, we establish the following embedding of RKHS of Mat\'ern kernels to H\"older spaces, via (i) norm equivalency between RKHS of a Mat\'ern-$\nu$ kernel and Sobolev space with order $\sobolevint$ and (ii) Sobolev embedding theorem that specifies the embedding of Sobolev space with order $\sobolevint$ to H\"older space with exponent $\alpha$. Note that~\citet{singh2021continuum} have shown that the minimax bandit algorithm over a Besov or Sobolev space is the same as one that is minimax over the smallest H\"older space that the Besov or Sobolev space embeds onto, although not explicitly for RKHS. For completeness, we still include the following proof. We first state the Sobolev embedding theorem~\citep[Theorem 5.4]{adams2003sobolev}.
\begin{theorem}[Sobolev embedding theorem~\citep{adams2003sobolev}]
    Let $\sobolevint$ be a non-negative integer. Suppose that the dimension $\paramdim<p \cdot \sobolevint$ and $\alpha = \sobolevint-\frac{\paramdim}{p}$. \highlight{Let $\Omega$ be a finite domain with Lipschitz boundary.} Then, the Sobolev space $\sobolevspace^{\sobolevint, p}(\Omega)$ is embedded onto H\"older space with exponent $\alpha$:
    \begin{equation}\label{eq: sobolev embedding}
        \sobolevspace^{\sobolevint, p}(\Omega) \subset \holderspace^{\alpha}(\Omega).
    \end{equation}
\end{theorem}
For our problem setting, \highlight{we set $p=2$ and $\paramdim=1$. The domain $\domain = [0,1]$ satisfies the Lipschitz boundary condition.} Therefore, $\sobolevspace^{\sobolevint}(\domain) \subset \holderspace^{\alpha}(\domain)$ where $\alpha = \sobolevint - \frac{1}{2}$. 
Combining Sobolev embedding theorem with the norm equivalency between Sobolev space and RKHS (Lemma~\ref{lemma: norm equivalency between rkhs and Sobolev}), we have the following result.
\begin{corollary}\label{thm: Matern RKHS embeds onto Holder space}
    Suppose that $\kernelfourier{\fourierrate}: \sR^\paramdim \times \sR^\paramdim \rightarrow \sR$ is a positive-definite translation-invariant kernel, whose Fourier transformation decays polynomially with rate $\fourierrate$, $\fourierrate>\paramdim/2, \fourierrate\in\mathbb{N}$. Then, the RKHS $\rkhs_{\kernelfourier{\fourierrate}}(\domain)$ is embedded onto H\"older space $\holderspace^{\alpha}(\domain)$ with exponent $\alpha=\fourierrate-\frac{\paramdim}{2}$:
    \begin{equation}
        \rkhs_{\kernelfourier{\fourierrate}}(\domain)\subset \holderspace^{\fourierrate-\frac{\paramdim}{2}}(\domain).
    \end{equation}
\end{corollary}
The above relationship is also studied in the earlier work of~\citet[Appendix~B.1]{shekhar2020multi}. 
Note that Mat\'ern kernels with regularity parameter $\nu$ have a Fourier decay rate of $\fourierrate = \nu+\frac{\paramdim}{2}$. Hence, $\rkhs_{\kernel_{\text{Mat\'ern}, \nu}}(\domain) \subset \holderspace^{\alpha}(\domain)$, for $\alpha = \nu$. Therefore, since UCB-Meta achieves on $\holderspace^{\alpha}(\domain)$ the regret rate of $\tilde{O}(T^\frac{\alpha+1}{2\alpha+1})$~\citep[Equation (19)]{liu2021smooth}, it achieves the same rate $\tilde{O}(T^\frac{\nu+1}{2\nu+1})$ on the subset $\rkhs_{\kernel_{\text{Mat\'ern}, \nu}}(\domain)$. Here, we omit the dependence on $B$, the RKHS norm bound. A function $\func\in \rkhs_{\kernel_{\text{Mat\'ern}, \nu}}(\domain, B)$ also has a finite H\"older norm $\Vert \func \Vert_{\holderspace^{\alpha=\nu}}$. The norm $\Vert \func \Vert_{\holderspace^{\nu}}$, by definition, poses an upper bound on $L$ (using the notation from~\citet[Definition 1]{liu2021smooth}, the H\"older-continuity coefficient of the $l$-th order derivative of $f$, where $l$ is the largest integer strictly less than $\alpha$. By Theorem 4 from~\citet{liu2021smooth}, we can see that $L$ affects the regret only through a multiplicative term and not through the exponents of $T$. Therefore, we omit the dependence on $B$ and write the regret rate of UCB-Meta as $\tilde{O}(T^\frac{\nu+1}{2\nu+1})$.

\subsection{Proof of Theorem~\ref{thm: CORRAL with kernelised bandits}}\label{sec: proof of thm CORRAL with kernelised bandits}
Recall that Theorem 5.3 in~\citet{pacchiano2020model} provides general regret bounds for CORRAL. The proof of our Theorem~\ref{thm: CORRAL with kernelised bandits} is an adaptation to the proof of Theorem 5.3 in~\citet{pacchiano2020model}.
We use the same notations as~\citet{pacchiano2020model} unless otherwise specified. $M$ is the number of base algorithms (also aligning with the statement in Theorem~\ref{thm: CORRAL with kernelised bandits}). $\delta$ is the probability of failure. $U: \sR \times [0,1] \rightarrow \sR^+$ is the cumulative regret function (for a base algorithm), such that $U(t, \delta)$ is the high-probability and anytime regret bound of a base algorithm. $\rho$ is the maximum of reciprocals of the probability that the base algorithm is chosen by the master algorithm over all time steps. $\eta$ is the learning rate of the master algorithm whose value is determined later in the proof. 

In Section~\ref{sec: kernelucb and gpucb}, we discussed briefly SupKernelUCB~\citep{valko2013finite} versus GP-UCB~\citet{srinivas2009gaussian}. Despite the convenient implementation and good empirical performance of GP-UCB, SupKernelUCB matches the non-adaptive lower bound in the dependence on T except log factors under the RKHS assumption and thus is minimax optimal while GP-UCB is not. {UCB-Meta~\citep{liu2021smooth} as shown in Theorem~\ref{thm: UCB-Meta regret in RKHS} is also minimax optimal in the dependence on $T$ except log factors for the Mat\'ern RKHS setting. For this subsection, however, we use SupKernelUCB as base algorithms, since the regret bound of SupKernelUCB has an explicit dependence on $B$, while for UCB-Meta the dependence on $B$ would rely on an implicit constant (see proof of Theorem~\ref{thm: UCB-Meta regret in RKHS} in Appendix~\ref{sec: proof of UCB-Meta regret in RKHS}).}  We set $\paramdim=1$ as specified in Section~\ref{sec: adaptive upper bounds}. 

Given $B$ and $\nu$ of a Mat\'ern-$\nu$ kernel, the regret bound of SupKernelUCB is $\tilde{O}(B^{\frac{1}{2}}T^\frac{\nu+1}{\nu+2})$ in the RKHS of the Mat\'ern kernel~\citep[Theorem 1]{valko2013finite}. Note that the original SupKernelUCB (i) is for finite action set and (ii) takes $T$ as input and therefore does not have any time regret guarantees. As mentioned in Section~\ref{sec: kernelucb and gpucb}, \citet{cai2021lower} argue that the aforementioned problem (i) could be extended to the continuum-armed setting by a discretization argument with an extra $O(d(log(T)))$ term in the regret. The problem (ii) can be theoretically circumvented by the doubling procedure~\citep{auer1995gambling}. Doubling converts an algorithm with (cumulative) regret bound for fixed $T$ to one with anytime regret bound, suffering only up to constant factors in the regret. \footnote{The doubling procedure is also used in other works that use CORRAL to adapt to unknown parameters of the function space, for example~\citet{liu2021smooth} which studied adaptivity to the H\"older exponent.} Therefore, for theoretical interest, we treat SupKernelUCB as the minimax optimal base algorithm with anytime regret upper bound $\tilde{O}(B^{\frac{1}{2}}T^\frac{\nu+1}{\nu+2}), \forall T$. 

We acknowledge that this is for theoretical convenience only and it remains an important open problem~\citep{vakili2021open} to improve the regret bound of the practical GP-UCB algorithm under RKHS assumptions.

We plug in $U(T, \delta) = \tilde{O}(B^\frac{1}{2}T^\frac{\nu+1}{2\nu+1})$  for the base algorithms for CORRAL. Following the proof of~\citet[Theorem 5.3]{pacchiano2020model}, we have the following. Note that this upper bound holds with respect to any base algorithm with anytime high-probability regret $U(t, \delta)$. Therefore, we plug in the regret of the best base algorithm, which is $U(t, \delta) = \tilde{O}({B^*}^\frac{1}{2}t^\frac{\nu^*+1}{2\nu^*+1})$ because $\nu^*, B^*$ belong in the set of candidate values $\vu$. 
\begin{align*}
    R_T&\leq O(\frac{M\ln(T)}{\eta}+T\eta) - \E\left[\frac{\rho}{40\eta\ln(T)}-\rho U(T/\rho, \delta)\log(T)\right] + \delta T + 8\sqrt{MT\log(\frac{4T M }{\delta})}\\
    &\leq \tilde{O}(\frac{M}{\eta}+T\eta+\delta T +\sqrt{M T}) - \E\left[\tilde{O}(\frac{\rho}{\eta}-\rho\sqrt{B^*}T^{\frac{\nu^*+1}{2\nu^*+1}}\rho^{-\frac{\nu^*+1}{2\nu^*+1}}) \right]\\
    &\stackrel{\text{set}~\delta=\frac{1}{T}}{=}\tilde{O}(\frac{M}{\eta}+T\eta+\sqrt{M T}) - \E\left[\tilde{O}(\frac{\rho}{\eta}-\sqrt{B^*}T^{\frac{\nu^*+1}{2\nu^*+1}}\rho^{\frac{\nu^*}{2\nu^*+1}}) \right]
\end{align*}
Maximizing the above equation over $\rho$ results in $\rho\propto \eta^\frac{2\nu^*+1}{\nu^*+1}{B^*}^\frac{\nu^*+\frac{1}{2}}{\nu^*+1}T$. 
If we plug this value for $\rho$ in the above equation, then the regret is bounded by:
\begin{align*}
    R_T&= \tilde{O}(\frac{M}{\eta}+T\eta+\sqrt{M T})  - \tilde{O}(\eta^\frac{\nu^*}{\nu^*+1}{B^*}^\frac{\nu^*+\frac{1}{2}}{\nu^*+1}T - \eta^\frac{\nu^*}{\nu^*+1} {B^*}^\frac{2\nu^*+1}{2\nu^*+2}T)\\
    &\leq\tilde{O}(\frac{M}{\eta}+T\eta+\sqrt{MT}+\eta^\frac{\nu^*}{\nu^*+1} {B^*}^\frac{2\nu^*+1}{2\nu^*+2}T)
\end{align*}
For the problem of adapting to kernel regularity (represented by $\nu^*$ when the kernel is a Mat\'ern kernel), since CORRAL does not have access to $\nu^*$ (and $B^*$), we choose $\eta$ with respect to the user-specified parameter $\tilde{\nu}$: $\eta = T^{-\frac{\tilde{\nu}+1}{2\tilde{\nu}+1}}$.
Plugging this choice of $\eta$ back in the above equation, we have:
\begin{align*}
    R_T\leq \tilde{O}(MT^\frac{\tilde{\nu}+1}{2\tilde{\nu}+1} + {B^*}^\frac{2\nu^*+1}{2\nu^*+2}T^{\frac{\tilde{\nu}\nu^*+2\tilde{\nu}+1}{(2\tilde{\nu}+1)(\nu^*+1)}}).
\end{align*}
Absorbing the dependence on $M$ and $B$ in $\tilde{O}$, we then have the regret rate in~\eqref{eq: corral upper bound}. 

\subsection{Proof of Theorem~\ref{thm: RBBE with kernelised bandit}}\label{sec: proof of thm RBBE with kernelised bandit}
The proof follows from the general form of regret upper bound of RBBE (Theorem 5.1 from~\citet{pacchiano2020regret}). The regret bound in Theorem 5.1 in~\citet{pacchiano2020regret} is expressed with the ``play ratio'' $\sum_{i\in\mathcal{B}} \frac{n_i(t_i)}{n_*(t_i)}$, where $\mathcal{B}$ denotes the set of misspecified base algorithms, $t_i$ denotes the last round before base algorithm $i$ is eliminated,  and $n_i(t)$ denotes the number of times $i$ is selected until time step $t \leq T$. In the following part, we use Lemma A.3 in~\citet{pacchiano2020regret} to calculate the play ratio, then plug it in Theorem 5.1 of~\citet{pacchiano2020regret} to get the final regret bound. For reasons why the more straightforward result (Theorem 5.4 in~\citet{pacchiano2020regret}) is not used, see the end of this subsection for an explanation. 

In the following, each base algorithm $i$ has the following candidate pseudo regret bound (equation (7) in~\citet{pacchiano2020regret}): 
\begin{equation}\label{eq: proof of RBBE, candidate regret form}
    R_i(t)\leq C \someparam_i T^{\beta_i}, 
\end{equation}
where $C\geq 1$ is some term independent of $T$ or $i$, and $\someparam_i \geq 1$ is some parameter dependent on $i$.   
For minimax optimal kernelised bandit algorithms instantiated with $\nu_i$ (parameter of the Mat\'ern kernel), $\beta_i = \frac{\nu_i + \paramdim }{2\nu_i + \paramdim}$.
We write down the general regret bound of RBBE here for completeness (Theorem 5.1~\citep{pacchiano2020regret}). Below, $*$ denotes any well-specified learner, that is, a leaner whose actual (pseudo) regret $\textit{Reg}_i$ is upper bounded by its candidate (which means if well-specified) regret bound $R_i(T)$. 
\begin{align*}
    R_T \leq &\sum_{i=1}^M R_*(n_*(t_i)) + \sum_{i\in\mathcal{B}}\frac{n_i(t_i)}{n_*(t_i)}R_*\left(n_*(t_i)\right) + 2M + 2c\sum_{i\in\mathcal{B}}\sqrt{n_i(t_i)\ln(\frac{M\ln(T)}{\delta})} \\
    & + 2c\sum_{i\in\mathcal{B}}\sqrt{\frac{n_i(t_i)}{n_*(t_i)}}\sqrt{n_i(t_i)\ln(\frac{M\ln(T)}{\delta})}
\end{align*}
We refer to the five terms in the above summation above as $\#1\dots \#5$. 

The terms $\#1+ \#3$ can be bounded the same way as in the proof of Theorem 5.4 in~\citet{pacchiano2020regret}:
\begin{align*}
    \sum_{i=1}^M R_*(n_*(t_i)) + 2M \leq M R_*(T) + 2M \leq \tilde{O}(M\someparam_*T^{\beta_*}).
\end{align*}

The term $\#4$ is bounded also following the proof in~\citet{pacchiano2020regret}:
\begin{align*}
    2c\sum_{i\in\mathcal{B}}\sqrt{n_i(t_i)\ln(\frac{M\ln(T)}{\delta})} &\leq 2c\sqrt{\vert \mathcal{B}\vert \ln\frac{M\ln(T)}{\delta}\sum_{i\in\mathcal{B}}n_i(t_i)} \\
    &\leq 2c\sqrt{\vert \mathcal{B}\vert T \ln\frac{M\ln(T)}{\delta}}
\end{align*}

Bounding the term $\#1$ and $\#5$, however, needs changes to the proof of Theorem 5.4~\citep{pacchiano2020regret}, since the play ratio is involved. Lemma A.3 in~\citet{pacchiano2020regret} states that for two base learners $i, j$,
\begin{equation}\label{eq: play ratio upper bound in RBBE, Lemma A.3}
    \frac{n_i(t)}{n_j(t)}\leq \max\left\{\left(2\frac{\someparam_j}{\someparam_i}\right)^\frac{1}{\beta_i} \left(n_j(t)\right)^{\frac{\beta_j}{\beta_i}-1}, 2\right\}.
\end{equation}
Therefore, the play ratio between a misspecified base learner $i$ and a well-specified leaner $*$ can be bounded by:
\begin{align*}
    \frac{n_i(t)}{n_*(t)}&\leq 2 + \left(2\frac{\someparam_*}{\someparam_i}\right)^\frac{1}{\beta_i} n_*(t)^{\frac{\beta_*}{\beta_i} - 1} \\
    &\leq 2 + 4C_2 B_* n_*(t)^{\frac{\beta_*}{\beta_i}-1} \\
    &\leq 2 + 4C_2 B_* n_*(t)^{2{\beta_*}-1}.
\end{align*}
The first inequality above is simply plugging $j=*$ (representing a well-specified learner), and using that $\max\{x, y\}\leq x+y$. 
For the second inequality, recall that the minimax optimal SupKernelUCB algorithm has a regret rate (if the kernel parameter $\nu$ and RKHS norm bound $B$ are known) of $\tilde{O}(\sqrt{B\gamma_T T}) = \tilde{O}(\sqrt{B}T^{\frac{\nu+\paramdim}{2\nu+\paramdim}})$. The $\tilde{O}$ notation hides polynomial terms that are dependent on $\log(T), \paramdim$. Therefore, the parameter $\someparam_i$ in~\eqref{eq: proof of RBBE, candidate regret form} that depends on the index of the base algorithm $i$ is $\theta_i\propto\sqrt{B_i}$. Given the assumption that $\theta_i\geq 1$, $\frac{\someparam_*}{\someparam_i} \leq C_1 \sqrt{B_*}$ for some constant $C_1$. Since $\beta_i \geq \frac{1}{2}$, $\left(2\frac{\someparam_*}{\someparam_i}\right)^\frac{1}{\beta_i}\leq 4C_2 B_*$ for some constant $C_2$.
Also in the last two inequalities, we used $\beta_i \geq \frac{1}{2}$, that is, every base algorithm used in Theorem~\ref{thm: RBBE with kernelised bandit} have at least $\tilde{O}({T^\frac{1}{2}})$ regret. Therefore, we have the following bound on the sum of play ratio:
\begin{align}
    \sum_{i\in\mathcal{B}} \frac{n_i(t)}{n_*(t)}&\leq 2\vert\mathcal{B}\vert + 4C_2 B_*\vert\mathcal{B}\vert (n_*(t))^{(2\beta_* - 1)}\\
    &\leq 2\vert\mathcal{B}\vert + 4C_2 B_*\vert\mathcal{B}\vert T^{(2\beta_* - 1)} = 2\vert\mathcal{B}\vert(1+ 2 C_2 B_* T^{(2\beta_* - 1)})\label{eq: play ratio intermediate RBBE}
\end{align}
We can plug equation~\eqref{eq: play ratio intermediate RBBE} to bound $\#5$ as follows.
\begin{align*}
    2c\sum_{i\in\mathcal{B}}\sqrt{\frac{n_i(t_i)}{n_*(t_i)}}\sqrt{n_i(t_i)\ln(\frac{M\ln(T)}{\delta})} &\leq 2c\sqrt{\sum_{i\in\mathcal{B}}\frac{n_i(t_i)}{n_*(t_i)}\sum_{i\in\mathcal{B}}n_i(t_i)\ln\frac{M\ln(T)}{\delta}}\\
    &\leq 2c\sqrt{\sum_{i\in\mathcal{B}}\frac{n_i(t)}{n_*(t)}T\ln\frac{M\ln(T)}{\delta}} \\
    &\leq 2c\sqrt{2\vert\mathcal{B}\vert(1+ 2C_2 B_* T^{(2\beta_* - 1)}) T\ln\frac{M\ln(T)}{\delta}} \\
    &=\tilde{O}(\vert\mathcal{B}\vert^{\frac{1}{2}} {B_*}^\frac{1}{2} T^{\beta_*})
\end{align*}

Similarly, the upper bound of term $\#2$ relies on~\eqref{eq: play ratio intermediate RBBE} as well. 
\begin{align*}
    \sum_{i\in\mathcal{B}}\frac{n_i(t_i)}{n_*(t_i)}R_*\left(n_*(t_i)\right) &\leq C \sum_{i\in\mathcal{B}}\frac{n_i(t)}{n_*(t)} \someparam_* n_*(t_i)\\
    &\leq C\sum_{i\in\mathcal{B}}\frac{n_i(t_i)}{(n_*(t_i))^{1-\beta_*}} \someparam_*\\
    &\leq C \left(\sum_{i\in\mathcal{B}}\frac{n_i(t_i)}{n_*(t_i)}\right)^{(1-\beta_*)} \someparam_* (n_i(t_i))^{\beta_*} \\
    &\leq C\someparam_* \left(2\vert\mathcal{B}\vert(1+ 2 C_2 B_* T^{(2\beta_* - 1)})\right)^{(1-\beta_*)} T^{\beta_*} \\
    &=\tilde{O}(\someparam_* \vert\mathcal{B}\vert^{(1-\beta_*)} {B_*}^{1-\beta^*}T^{(2\beta_*-1)(1-\beta_*)+\beta_*}) \\
    &=\tilde{O}(\someparam_* \vert\mathcal{B}\vert^{(1-\beta_*)} {B_*}^{1-\beta^*} T^{4\beta_* + 2\beta_*^2 - 1})
\end{align*}

Now that the asymptotic rates of the five terms are derived, we can see that term $\#2$ dominates in the dependence of $T$ and $\#5$ dominates dependence on $\vert\mathcal{B}\vert, B_*$, and hence, the regret of RBBE can be bounded as follows.
\begin{align}
    R_T &\leq \tilde{O}(\someparam_* \vert\mathcal{B}\vert^{\frac{1}{2}} {B_*}^{\frac{1}{2}} T^{4\beta_* + 2\beta_*^2 - 1}) \\
    &=\tilde{O}(\someparam_* \vert\mathcal{B}\vert^{\frac{1}{2}} {B_*}^{\frac{1}{2}} T^{\frac{2\nu_*^2+4\nu^*+1}{(2\nu^*+1)^2}})\\
    &=\tilde{O}(\someparam_* M^{\frac{1}{2}}{B_*}^{\frac{1}{2}} T^{\frac{2\nu_*^2+4\nu^*+1}{(2\nu^*+1)^2}})
\end{align}

Finally, the reason for not using the straightforward results in Theorem 5.4 of~\citet{pacchiano2020regret} is as follows. In adaptation to the kernel regularity parameter $\nu$, the candidate regret bounds of base algorithms do \emph{not} have the same exponent of $T$. The candidate regret bounds having the same rates of $T$ is a requirement for the more straightforward results, hence, those results are not directly applicable to our setting. 
\end{document}